%% file: polyak_arxiv.tex
\documentclass[12pt]{article}

\usepackage[letterpaper,left=1in,right=1in,top=1in,bottom=1in]{geometry}
\usepackage[utf8]{inputenc}
\usepackage[T1]{fontenc}
\usepackage[scaled=.9]{helvet}
\usepackage{xcolor}
\usepackage{subcaption} 
\usepackage{amsfonts}
\usepackage{nicefrac}
\usepackage{microtype}
\usepackage{url}
\usepackage{array}
\usepackage{colortbl}
\usepackage{makecell}
\usepackage{natbib}
\usepackage{authblk}
\input{math_commands.tex}

\input{arxiv_version}

\def\x{{\mathbf{x}}}

\def\y{{\mathbf{y}}}

\def\g{{\mathbf{g}}}

\def\w{{\mathbf{w}}}

\def\y{{\mathbf{y}}}
\def\I{{\mathbf{I}}}
\def\H{{\mathbf{H}}}
\def\J{{\mathbf{J}}}

\newcommand{\reals}{\mathbb{R}}
\newcommand{\brac}[1]{\left(#1\right)}
\newcommand{\enorm}[1]{\left\Vert#1\right\Vert}
\usepackage{bbm}

\def\muglobal{{\mu_{\mathrm{G}}}}

\def\PAMOO{\texttt{PAMOO}}
\def\MG{\mathrm{MG}}
\def\MGAMOO{\texttt{MG-AMOO}}

\def\EW{\texttt{EW}}
\def\CAMOO{\texttt{CAMOO}}

\title{Simple Optimizers for Convex Aligned Multi-Objective Optimization}
\author[2]{Ben Kretzu}
\author[1]{Karen Ullrich}
\author[1]{Yonathan Efroni}
\affil[1]{Meta AI}
\affil[2]{Technion}

\begin{document}

\maketitle
\begin{abstract}
It is widely recognized in modern machine learning practice that access to a diverse set of tasks can enhance performance across those tasks. This observation suggests that, unlike in general multi-objective optimization, the objectives in many real-world settings may not be inherently conflicting. To address this, \citet{efroni2025aligned} introduced the Aligned Multi-Objective Optimization (AMOO) framework and proposed gradient-based algorithms with provable convergence guarantees. However, their analysis relies on strong assumptions, particularly strong convexity, which implies the existence of a unique optimal solution. In this work, we relax this assumption and study gradient-descent algorithms for convex AMOO under standard smoothness or Lipschitz continuity conditions—assumptions more consistent with those used in deep learning practice. This generalization requires new analytical tools and metrics to characterize convergence in the convex AMOO setting. We develop such tools, propose scalable algorithms for convex AMOO, and establish their convergence guarantees. Additionally, we prove a novel lower bound that demonstrates the suboptimality of naive equal-weight approaches compared to our methods.
\end{abstract}

\input{paper/main_paper}

\bibliographystyle{plainnat}
\bibliography{citation}

\newpage
\appendix

\input{paper/appendix}

\end{document}

%% file: math_commands.tex
\usepackage{amsmath,amsfonts,bm}

\def\eqref#1{equation~\ref{#1}}

\def\1{\bm{1}}

\def\mC{{\bm{C}}}

\DeclareMathAlphabet{\mathsfit}{\encodingdefault}{\sfdefault}{m}{sl}
\SetMathAlphabet{\mathsfit}{bold}{\encodingdefault}{\sfdefault}{bx}{n}

\DeclareMathOperator*{\argmax}{arg\,max}

%% file: arxiv_version.tex
\usepackage{times}
\usepackage{amsmath, amsfonts, bm}
\usepackage{amsthm}
\usepackage{graphicx}
\usepackage[dvipsnames]{xcolor}
\usepackage{mathtools}
\usepackage{amssymb}
\usepackage{pifont}

\usepackage{makecell}
\usepackage[utf8]{inputenc}
\usepackage[T1]{fontenc} 
\usepackage{mathrsfs} 
\usepackage{xspace}
\usepackage{algorithm}
\usepackage{array}
\usepackage{multirow}
\usepackage{color}
\usepackage[english]{babel}
\usepackage{natbib}
\usepackage{wrapfig}
\usepackage{epstopdf}
\usepackage{url}
\usepackage{algorithmic}
\usepackage{bbm}
\usepackage{comment}
\usepackage{tikz}
\usepackage{booktabs}
\usepackage[flushleft]{threeparttable}
\usepackage{enumitem}
\usepackage{mdframed}

\allowdisplaybreaks

\usepackage{hyperref}
\usepackage[capitalize,noabbrev]{cleveref}
\usetikzlibrary{arrows}
\usepackage{thmtools}
\usepackage{thm-restate}

\newtheorem{theorem}{Theorem}
\newtheorem*{theorem*}{Theorem}
\newtheorem{lemma}{Lemma}
\newtheorem{definition}{Definition}

\newtheorem{corollary}{Corollary}

\mdtheorem{problem}{Problem}

\definecolor{yxc}{RGB}{255,0,0}
\definecolor{yjc}{RGB}{190,0,255}
\definecolor{whz}{RGB}{0,155,0}

\usepackage{enumitem}
\setlist[itemize]{leftmargin=*}
\setlist[enumerate]{leftmargin=*}

\ifdefined\usebigfont

\usepackage[fontsize=13pt]{scrextend}
\AtBeginDocument{\newgeometry{letterpaper,left=1.56in,right=1.56in,top=1.71in,bottom=1.77in}}
\else
\fi

%% file: paper/main_paper.tex
\section{Introduction}


Many real-world applications involve optimizing multiple objectives rather than focusing on a single objective. These problems have been extensively studied in the field of multi-objective optimization (MOO), by characterizing tradeoffs between competing objectives, namely, studying the Pareto front, or finding solutions that optimize particular desirable outcomes. In machine learning, MOO has gained significant traction over the past decades, especially in the paradigm of multi-task learning~\citep{caruana1997multitask,teh2017distral,sener2018multi,yu2020gradient,liu2021conflict,navon2022multi,achituve2024bayesian}, or optimizing in the presence of proxy tasks~\citep{bastani2021predicting,woodworth2023two,shen2023proxybo}.  These resulted with algorithmic advancements in computer vision~\citep{yuan2012visual,zhang2014facial,liu2019end}, large language models~\citep{uesato2022solving,lightman2023let,yang2024metaaligner,wang2024conditional,guo2025deepseek,team2025kimi}, and reinforcement learning~\citep{jaderberg2016reinforcement,teh2017distral,yu2020meta, veeriah2019discovery,dann2023reinforcement, d2024sharing} when trained with multiple reward functions to name only a handful. For example, in computer vision tasks, researchers seek representations that jointly optimize different tasks such as depth estimation and segmentation. More recently, in LLM post-training, algorithm designers incorporate multiple reward functions during the post-training phase, optimizing for attributes like response length, and alignment to different reward functions.


These examples reveal a recurring theme: in many machine learning applications, algorithm designers supply multiple objectives that are not necessarily in conflict with each other, but rather help to better specify the desired solution. This leads to a natural question: Can we develop optimizers that benefit from multi-objective feedback when the objectives are aligned? Recently,~\citet{efroni2025aligned} formalized the Aligned MOO (AMOO) setting, where the set of objectives is assumed to have a shared optimal solution, and designed gradient-descent algorithms for their framework. However, their work relied on several strict assumptions that often do not hold in practice. Arguably, the most restrictive one is requiring a variant of the strong convexity assumption—a condition rarely satisfied in practical deep learning applications--to hold. Such assumption implies the existence of a \textit{unique} solution that optimizes all objective functions. Yet, in common deep learning practice, the loss surface tends to be flat, namely, have small curvature \citep{li2018visualizing,petzka2021relative}. Further, practitioners can directly optimize a sharpness aware loss term to reduce the flatness further~\citep{foret2020sharpness}.  This suggest that in practice the strong convexity parameter may be arbitrarily small.



In this work, we develop provable and simple gradient-descent algorithms for AMOO under milder and more realistic assumptions compared to prior work. We avoid strong convexity assumptions, instead requiring only that functions be convex with either bounded Lipschitz constants or smoothness parameters. This is motivated by the fact that other common optimizers for single objective optimization in deep learning literature were developed for similar function classes~\citep{zinkevich2003online,duchi2011adaptive,kingma2014adam,defazio2024road}. Removing the strong convexity assumption in the AMOO setting introduces new challenges. Since there is no unique optimal solution, we cannot measure convergence to a specific point $\x_\star$, unlike in prior work. Instead, we propose an alternative metric to characterize convergence in the convex AMOO setting: the Maximum Gap (MG), defined as  $\MG(\x) := \max_{i\in [m]} \brac{f_i(\x) - \min_{\x} f_i(\x)}$. Our goal is to achieve fast convergence to a solution where $\MG(\x)=0$ . This metric is particularly appropriate for AMOO since in this framework functions can be optimized simultaneously and, hence, the MG should approach zero.


\begin{table}[t]
\centering
\renewcommand{\arraystretch}{1.25}  
\begin{tabular}{|>{\centering\arraybackslash}p{5cm}|>{\centering\arraybackslash}p{4.5cm}|}
\hline
\rowcolor{gray!50} \makecell{Algorithm} & \makecell{Convergence of $\MG(\bar{x})$} \\
\hline
 \makecell{\EW} &  $\Omega(\sqrt{m}G/\sqrt{K})$ \\
\hline
 \makecell{\PAMOO} & $\mathcal{O}(G/\sqrt{K})$\\
\hline
 \makecell{\MGAMOO\ w Polyak or GD} &   $\mathcal{O}(G/\sqrt{K})$\\
\hline
 \makecell{\MGAMOO\ w Online Learner} &   $\mathcal{O}(\mathrm{Regret}(K)/K)$\\
\hline
\end{tabular}
\vspace{0.1cm}
\caption{Main results introduced in this work that characterize the convergence of the maximum gap metric, $\MG(\bar{x})$. Here we provide the results for the $G$-Lipchitz function class; we include additional results for $\beta$-smooth functions in Section~\ref{subsec:mg-amoo}. $K$ is number of iterations, $m$ number of objectives, $\mathrm{Regret}(K)$ represents the regret of an online learner used in the meta-algorithm we design. 
}
\label{tab:intro}
\end{table}

We begin by establishing a surprising lower bound for the naive Equal-Weight (EW) algorithm (Section~\ref{sec:mg_and_lower_bound}), which applies GD to the sum of objective functions. Our analysis reveals that the MG metric may converge with polynomial dependence on the number of functions—an undesirable outcome implying that convergence deteriorates when incorporating additional loss functions. We then develop specialized Gradient-Descent (GD) methods for the AMOO setting that achieve faster convergence rates for the MG metric, independent of the number of functions (Section~\ref{sec:amoo_algorithms}). We provide a new analysis of the \PAMOO\ algorithm, and then design \MGAMOO, a meta-algorithm that uses only a single-objective optimizer to solve an AMOO problem.  Instantiating it with an online learner, Polyak step-size design or GD algorithms leads to new, scalable and simple algorithms for AMOO (see Table~\ref{tab:intro}). To the best of our knowledge, this represents the first non-asymptotic provable advantage of algorithms specifically designed for AMOO compared to the naive EW method. We conclude by conducting experiments with the different algorithms we introduced and comparing their performance to the naive EW as well as comparing their running time (Section~\ref{sec:experiments}).

\section{Preliminaries: Aligned Multi Objective Optimization} 


We describe the AMOO setting and the structural assumptions used in this work. Assume access to multi-objective feedback in the form of a vector-valued function
$
F: \reals^n \to \reals^m, \quad
F(\x) = \brac{f_1(\x), f_2(\x), \ldots, f_m(\x)},
$
and the optimal values of the individual objectives is denoted by $f_i^\star:=\min_{\x \in \mathbb{R}^n} f_i(\x)$ for $i\in [m]$.
Given a weight vector $\w\in \mathbb{R}_+^m$ we let $f_{\w}(\x) := \w^T F(\x)$ be the weighted function by $\w$. In the AMOO setting, we further assume that the objectives are \emph{aligned}, in the sense that they can be minimized simultaneously. That is, there exists a non-empty set of points $\mC_\star=\{ \x \in \reals^n | \ f_i(\x) = f_i^\star, ~~ \forall i\in[m] \} $ that minimizes all functions $\{f_i(\x)\}_{i\in [m]}$. Namely, for any $\x_\star \in \mC_\star$ it holds that
\begin{align*}
    \x_\star \in \arg\min_{\x\in \reals^n} f_i(\x) \quad \forall i \in [m]. 
\end{align*}

In this work we consider the convex setting, where each function $f_i$ is convex. We also assume standard structural properties of the objectives smoothness and Lipschitz continuity, which are commonly made in the literature~\citep{zinkevich2003online,duchi2011adaptive,kingma2014adam,defazio2024road}. Let $\enorm{\cdot}$ denote the Euclidean $L_2$ norm. A function is called $G$-Lipschitz if for all $\x \in \reals^n$, it holds that $\enorm{\nabla f(\x)} \leq G$. A function $f : \reals^n \to \reals$ is said to be $\beta$-smooth if for all $\x, \y \in \reals^n$:
$
f(\y) \leq f(\x) + \nabla f(\x)^\top (\y - \x) + \frac{\beta}{2} \enorm{\x - \y}^2.
$ Every $\beta$-smooth function has the following useful consequence.
\begin{lemma}[Standard result, E.g., 9.17~\citet{boyd2004convex}] \label{lemma:smooth-gradient-norm}
    Let $f : \reals^n \to \reals$ a $\beta$-smooth over $\reals^n$, and let $f^\star := \min_{\x \in \reals^n} ~ f(\x)$.  Then, for every $\x \in \reals^n$ it holds that
    $
        \enorm{\nabla f (\x)}^2 \leq 2\beta \brac{f(\x) - f^\star}.
    $
\end{lemma}

\paragraph{Additional notations:} We denote by $\Delta_m$ as the m-dimensional simplex and by $\mathbf{e}_{i}$ as the one-hot vector on the $i^{th}$ index, $\mathbf{e}_{i,j}=1$ if $i=j$ and  $\mathbf{e}_{i,j}=0$ otherwise.

\section{Prior Work} 
Multi-objective optimization~\citep{gunantara2018review} has a rich research history and found applications in many different domains. In machine learning applications, algorithm designers often apply such paradigm when few sources of feedback are available in multi-task learning or in the presence of proxy loss functions. From optimization and machine learning perspectives, most of previous research efforts have been devoted to approximating the Pareto front of optimal solutions or finding good balance between the different objectives. In this work, we focus on the AMOO setting and aim to develop improved algorithms under the alignment assumption.

The AMOO setting that we consider in this work was introduced by~\citet{efroni2025aligned}. Additionally on making the alignment assumption between the multiple objectives, the authors required a strong convexity parameter $\muglobal:= \min_{\x}\max_{\w\in \Delta_m}\lambda_{\min}\brac{\nabla^2 f_{\w}(\x)}$ to be positive $\muglobal>0$. Such assumption also implies there exist a unique $\x_\star$ that optimizes $\{ f_i\}_{i\in [m]}$ which enables to measure converence of algorithms by measuring their convergence to $\x_\star$. The authors designed the \CAMOO\ and \PAMOO\ algorithms and established both converge in linear rate to $\x_\star$ in the parameter $\muglobal/\beta$.   

Extending their result to the convex setting requires new analysis and perspective on AMOO. First, \CAMOO\ is not a valid approach if $\muglobal=0$. \CAMOO\ chooses the weighting vector in the $k^{\mathrm{th}}$ iteration as $\w_k\in \arg\max_{\w\in \Delta_m}\lambda_{\min}\brac{\nabla^2 f_{\w}(\x_k)}$, however, for the non-strongly convex setting, such approach is expected to fail. For general convex functions the solution of such optimization problem can be arbitrary: we can find a set of convex functions for which for any weight vector $\w$ it holds that $\lambda_{\min}\brac{\nabla^2 f_{\w}(\x)}=0$, and, hence, for such example, the solution can be the entire simplex~(see Appendix~\ref{app:camoo_is_invalid} for an example). This suggests \CAMOO\ is flawed for the non-strongly convex setting. Additionally, and more crucially, in the convex setting, without strong convexity assumption, it does not hold that there exists a \emph{unique} $\x_\star$. Hence, establishing convergence of $\enorm{\x_\star - \x_k}$ is no longer possible, and one is required to rethink of the metrics to characterize convergence. In this work, we provide an analysis of \PAMOO\ while removing the strong convexity assumption, as well as introducing a computationally cheap meta-algorithm for the AMOO framework.

\section{The Maximum-Gap Metric \& A Lower Bound}\label{sec:mg_and_lower_bound}

Without the strong convexity assumption, we can no longer establish point-wise convergence guarantees. What metric should we use in this scenario? Since all objectives can be jointly optimized, we measure convergence via the \textit{Maximum Gap (MG)} across the different objectives, defined as:
\[
\MG(\x) := \max_{i\in [m]} \brac{f_i(\x) - f_i^\star}.
\]
The MG metric characterizes the worst-case convergence across all functions; as such, it captures a robustness aspect of the solution. This criterion is also explored in the field of robust optimization~\cite{ben2009robust}, and typically leads to more challenging optimization problems~\cite{ben2002robust,wiesemann2013robust,chen2017robust}. However, under the alignment or approximate alignment assumptions, we may hope to obtain convergence guarantees with simple GD algorithms with respect to this robust criterion. Indeed, under such assumptions, we expect the MG to become negligible as the algorithm converges—namely, $\lim_{k\rightarrow \infty}\MG(\x_k)\sim 0$. 
Importantly, this should hold without requiring strong convexity, since convergence is measured with respect to the optimal values, as opposed to point-wise convergence to a single optimal solution.

Next, we focus on the naive Equal-Weights (\EW) baseline, which applies a GD algorithm to a single objective that uniformly weights all component objectives:
\[
f_{\EW}(\x) = \frac{1}{m} \sum_{i\in [m]}  f_i(\x).
\]
This approach is arguably the simplest baseline for multi-objective and multi-task problems~\citep{teh2017distral,sener2018multi,yu2020gradient}. Notably, the \EW\ baseline also produces a solution that minimizes the MG metric, as the following simple claim shows (see Appendix~\ref{app:x_ew_minimizes_mg}).

\begin{restatable}{claim}{ClaimOne}
\label{clm:1}    
    Assume $\{f_i\}_{i\in [m]}$ are aligned. Let $\x^\star_{\EW}\in \arg\min_{\x\in \mathbb{R}^n} f_{\EW}(\x)$. Then, $\MG(\x^\star_{\EW})=0.$
\end{restatable}

 Although \EW\ optimizes the MG metric asymptotically, one may ask how the convergence rate is effected by the number of objective, and whether it can be improved by alternative algorithms. In what follows, we provide a lower bound showing that, for \EW, the MG metric may converge with polynomial dependence on the number of functions. This indicates a significant disadvantage of the \EW\ approach. We focus on a variant of the method that selects the step size via the Polyak step-size design which is adaptively chooses the learning rate~\cite{polyak1987introduction,hazan2019revisiting}. The gradient-based algorithm outputs $\bar{\x} = \frac{1}{K} \sum_{k=1}^K \x_k$, where the iterates are generated by:
\begin{align}\label{eq:ew_update_step}
 \x_{k+1} = \x_k - \eta_k \nabla f_{\EW}(\x_k),\  \text{with}\ \eta_k = \frac{f_{\EW}(\x_{k}) - f_{\EW}(\x_{\EW}^\star)}{\enorm{\nabla f_{\EW}(\x_k)}^2}.
\end{align}

\begin{theorem}[Lower Bound for \EW\ with Polyak Step Size]\label{thm:lower_bound_ew_polyak}
    There exists a set of functions $\{f_i(\x)\}_{i\in [m]}$ that are $G$-Lipschitz and convex such that, for $K \leq m$ iterations, the worst-case gap that Equal Weights with the Polyak algorithm guarantees depends polynomially on $m$. Specially, the output of GD with Polyak step size converges with rate $\MG(\bar{\x}) \in \Omega \brac{\sqrt{m}G\enorm{\x_1-\x_{\EW}^\star}/\sqrt{K}}$. 
\end{theorem}

\begin{proof}[Proof of Theorem~\ref{thm:lower_bound_ew_polyak}]
Let $\epsilon > 0 $, $n \geq m > 1$, and $a=(m-1)\epsilon$. Let $f_i(\x) = | x_i |$ for every $i \in [m]$, which are $1$-Lipschitz ($G=1$), and convex functions, and initial point $\x_1 = (x_1, x_2, \dots, x_n) = (a, \epsilon, \dots, \epsilon, 0, \dots, 0) \in \reals^n$, i.e. $x_1 = a$, $x_i = \epsilon$ for every $i \in [2,m]$, and $x_i = 0$ for every $i \in [m+1,n]$. Note that, the subgradient is $\nabla f_i(\x) = \mathrm{sign}(x_i)$ for every $i \in [m]$ and for every point that $x_i \neq 0$. When $x_i = 0$ the subgradient is $\nabla f_i(\x) \in [-1,1]$.
\\

Observe that, from \eqref{eq:ew_update_step}, the update rule is given by
\begin{align}
    \x_{k+1} & = \x_k - \frac{\sum_i f_i(\x_k) - f_i^\star}{\enorm{\sum_i \nabla f_i (\x_k)}^2} \sum_i \nabla f_i (\x_k),\label{eq:update_step_ew_lower_bound}
\end{align}
since 
$$
\eta_k \nabla f_{\EW}(\x_k) = \frac{m^{-1} \sum_i f_i(\x_k) - f_i^\star}{\enorm{m^{-1} \sum_i \nabla f_i (\x_k)}^2} m^{-1} \sum_i \nabla f_i (\x_k) = \frac{\sum_i f_i(\x_k) - f_i^\star}{\enorm{\sum_i \nabla f_i (\x_k)}^2} \sum_i \nabla f_i (\x_k).
$$
From \eqref{eq:update_step_ew_lower_bound} The update step is:
\begin{align*}
    \x_{k+1} = \x_k - \frac{\sum_{i\in [m]} |x_{k,i}|}{m} \brac{\mathrm{sign}(x_{k,1}), \dots, \mathrm{sign}(x_{k,m}), 0, \dots, 0}.
\end{align*}
Using Lemma~\ref{lemma:induction_lower_bound} (see Appendix~\ref{app:induction_lower_bound}) it holds that after $k$ iterations the first component of the iterate, denoted by $x_{k,1}$ is:
\begin{align*}
    x_{k,1} = a\brac{1-\frac{2}{m}}^{k-1}.  
\end{align*}
Recall that the return point is $\bar{\x} = \frac{1}{K} \sum_{k=1}^K \x_k$. Which holds that the first component of it is:
\begin{align*}
    \bar{x}_{1} & = \frac{a}{K} \sum_{k=1}^K \brac{1-\frac{2}{m}}^{k-1} = \frac{a}{K} \frac{1 - \brac{1-\frac{2}{m}}^K}{1 - \brac{1-\frac{2}{m}}} = \frac{a}{K} \frac{1 - \brac{1-\frac{2}{m}}^K}{\frac{2}{m}} = \frac{ma}{2K} \brac{1 - \brac{1-\frac{2}{m}}^K}.  
\end{align*}
By using \eqref{eq:add_res_upper_bound_for_lower_bound}, it holds that
\begin{align*}
    \bar{x}_{1} & = \frac{ma}{2K} \brac{1 - \brac{1-\frac{2}{m}}^K} \geq \frac{ma}{2K} \frac{2K}{m+2K} = \frac{ma}{m+2K}.  
\end{align*}
Recall that $a>0$, and $K>0$. Then the RHS of the last inequality, i.e. $\frac{ma}{m+2K}$, is a decreasing function by $K$.
Thus, for $K \leq m$ iterations we have that 
\begin{align*}
    \max_{i\in[m]} f_i(\bar{\x}) - f_i^\star \geq f_1(\bar{\x}) - f_1^\star = |\bar{x}_{1}| - 0 \geq \frac{ma}{m+2K} \geq \frac{\sqrt{m}\sqrt{K}a}{3K} = \frac{\sqrt{m}a}{3\sqrt{K}} .
\end{align*}
Since $\x_{\EW}^\star = \mathbf{0}$, it holds that $\enorm{\x_1 - \x_{\EW}^\star} = \sqrt{a^2 + (m-1)\epsilon^2} = \sqrt{a^2 + \frac{a^2}{m-1}} = a\sqrt{ \frac{m}{m-1}}$. Then, we obtain that the max-gap is lower bounded by:
\begin{align*}
    \MG(\bar{\x}) = \max_{i\in[m]} f_i(\bar{\x}) - f_i^\star \geq \frac{\sqrt{m}a}{3\sqrt{K}} = \frac{\sqrt{m-1}\enorm{\x_1 - \x_{\EW}^\star}}{3\sqrt{K}} = \sqrt{\frac{m-1}{9}}\frac{G \enorm{\x_1 - \x_{\EW}^\star}}{\sqrt{K}}.
\end{align*}
\end{proof}

\section{Simple Optimizers for Convex AMOO}\label{sec:amoo_algorithms}


In the previous section, we introduced the MG metric as a meaningful objective for optimization in the general convex AMOO setting, without relying on strong convexity assumptions. We also established a lower bound for the \EW\ algorithm, showing that the MG may converge at a polynomial rate in the number of objective functions.  In this section, we study specialized GD methods for AMOO and establish their convergence rates are independent of the number of objectives. First, in Section~\ref{subsec:pamoo}, we provide a new analysis of the Polyak Aligned Multi-Objective Optimizer (\PAMOO). Then, in Section~\ref{subsec:mg-amoo}, we introduce an algorithmic reduction from AMOO to single-objective optimization. Instantiating this reduction with standard optimization methods yields simple and efficient algorithms for AMOO, which enjoys best of both worlds. We conclude by comparing the computational complexity of existing approaches in Section~\ref{sec:computational} and show that the new \MGAMOO\ method is significantly more efficient than \PAMOO\ and comparable in cost to the \EW\ baseline. 
\begin{figure}[t]
\centering
\begin{minipage}{0.53\textwidth}
\begin{algorithm}[H]
 \caption{\PAMOO}\label{alg:PAMOO_original} 
\begin{algorithmic}
  \STATE \textbf{initialize:} $\{ f_i^\star \}_{i\in [m]}$, $K$.
  \FOR{$k=1,2,\ldots,K$}
 \STATE $\Delta_{k} := \left[  f_{1}(\x_k)-f_{1}^\star,\dots, f_{m}(\x_k)-f_{m}^\star\ \right]$
 \STATE $\J_k := \left[ \nabla f_1(\x_k) \dots \nabla f_m(\x_k)  \right] \in \reals^{n \times m}$
  \STATE $\w_k\in \arg\max_{\w\in \reals^m_+} 2\w^\top \Delta_k - \w^\top \J_k^\top\J_k \w$
  \STATE $\x_{k+1} = \x_k - \nabla f_{\w_k}(\x_k)$ 
  \ENDFOR
  \STATE \textbf{return:} $\bar{\x} = \frac{1}{K} \sum_{k=1}^K \x_k$ 
\end{algorithmic}
\end{algorithm}
\end{minipage}
\hfill
 \begin{minipage}{0.45\textwidth}
\vspace{-0.8cm}
\begin{algorithm}[H]
 \caption{\MGAMOO}
 \label{alg:MGAMOO-Reduction} 
\begin{algorithmic}[1]
  \STATE \textbf{initialize:} $\{ f_i^\star \}_{i\in [m]}$, $K$, \texttt{SOO} 
  \FOR{$k=1,2,\ldots,K$}
 \STATE $I(k) \in \argmax_{i\in[m]} f_i(\x_{k}) - f_i^\star$
 \STATE $\x_{k+1} = \texttt{SOO}(f_{I(k)}, x_k)$
  \ENDFOR
  \STATE \textbf{return:} $\bar{\x} = \frac{1}{K} \sum_{k=1}^K \x_k$ 
\end{algorithmic}
\end{algorithm}
\end{minipage}
\end{figure}

\subsection{Polyak Method for AMOO (\PAMOO)} \label{subsec:pamoo}

The Polyak step-size design~\cite{polyak1987introduction, hazan2019revisiting} provides a parameter-free way to choose the learning rate in single-objective optimization problems. This approach was generalized to the AMOO setting in~\citet{efroni2025aligned}, resulting in the \PAMOO\ algorithm (Algorithm~\ref{alg:PAMOO_original}). Its derivation is motivated by arguments similar to those behind the original Polyak step-size. 

Assume the update rule of $\x_k$ has the form $\x_{k+1} = \x_k - \nabla f_\w(\x_k)$ for a fixed positive weight vector $\w \in \mathbb{R}_+^m$. Then, due to the convexity of $f_\w$, the following inequality holds:
\begin{align}
    \enorm{\x_{k+1}-\x}^2 \leq \enorm{\x_k-\x}^2 - 2 \brac{f_{\w}(\x_k)- f_{\w} (\x)} +\enorm{\nabla f_{\w}(\x_k)}^2. \label{eq:PAMOO_main_lower_bound}
\end{align}
Choosing $\w_k \in \arg\max_{\w \in \mathbb{R}_+^m} 2\brac{f_{\w}(\x_k) - f_{\w}(\x)} - \enorm{\nabla f_{\w}(\x_k)}^2$ to maximize the decrease results in the \PAMOO\ method.

Prior work established convergence of \PAMOO\ under strong convexity and self-concordance assumptions, which often do not hold in practical settings. The following result, Theorem~\ref{thm:original_pamoo_convergence}, establishes convergence of \PAMOO\ in the convex setting, removing the previously strict assumptions. Moreover, unlike \EW, the convergence of \PAMOO\ is independent of the number of objective functions. The proof of the theorem appears in Appendix~\ref{app:original_pamoo}.

\begin{restatable}[Convergence of \PAMOO]{theorem}{OriginalPAmooConvergence}
\label{thm:original_pamoo_convergence}            
    Algorithm~\ref{alg:PAMOO_original} guarantees that after $K$ iterations, the worst-case gap is upper bounded by:
    \begin{align*}
        (1) \quad & \max_{i\in[m]} f_i(\bar{\x}) - f_i^\star \leq \frac{ G \enorm{\x_1 - \x_\star} }{\sqrt{K}}, \quad \text{when $\{ f_i(\x) \}_{i\in [m]}$ are $G$-Lipschitz.} \\
        (2) \quad & \max_{i\in[m]} f_i(\bar{\x}) - f_i^\star \leq \frac{2\beta \enorm{\x_1 - \x_\star} }{K}, \quad \text{when $\{ f_i(\x) \}_{i\in [m]}$ are $\beta$-smooth.}
    \end{align*}
\end{restatable}

This result establishes convergence of the MG metric at a rate that is independent of the number of objective functions, representing a substantial improvement over the \EW\ method, which converges at a rate polynomial in the number of objectives (Theorem~\ref{thm:lower_bound_ew_polyak}).

\subsection{Reduction Approach: Maximum-Gap AMOO (\MGAMOO)}\label{subsec:mg-amoo}

In this section, we introduce a meta-algorithm for the AMOO setting that uses a single-objective optimizer in a black-box manner and returns a near-optimal solution for the AMOO problem. Specifically, we design a reduction from AMOO to single-objective optimization. We show that any online learning algorithm with vanishing regret can be used to solve an AMOO instance.

We refer to this algorithmic reduction as Maximum Gap AMOO (\MGAMOO), see Algorithm~\ref{alg:MGAMOO-Reduction}. The algorithm assumes access to a single-objective optimizer (\texttt{SOO})—a procedure that receives the current iterate $\x_k$, a function $f$, and returns an updated iterate $\x_{k+1}$. Similar to \PAMOO, it also assumes access to the optimal values $\{f_i^\star\}_{i \in [m]}$. At each iteration, the function with the largest gap is selected, and \texttt{SOO} is called to produce the next iterate.

This reduction is motivated by the following simple lemma:


\begin{lemma}[Reduction: AMOO to Single-Objective Optimization]\label{lemma:worst_case_gap_avg}
    Let $I(k) \in \argmax_{i\in[m]} f_i(\x_{k}) - f_i^\star$, and let $\{ \x_k \}_{k=1}^K$ be the optimization trajectory. Any algorithm that returns the average iterate $\bar{\x} = \frac{1}{K} \sum_{k=1}^K \x_k$ guarantees the worst-case max-gap is upper bounded by
    \begin{align}
        \MG(\bar{\x}) := \max_{i\in[m]} f_i(\bar{\x}) - f_i^\star \leq \frac{1}{K}\sum_{k=1}^K \brac{f_{I(k)}(\x_k) - f_{I(k)}^\star}. \label{eq:lemma_2_main_lemma}
    \end{align}
\end{lemma}

\begin{proof}[Proof of Lemma~\ref{lemma:worst_case_gap_avg}]
    Denote $I_e \in \argmax_{i\in[m]} f_i(\bar{\x}) - f_i^\star$, the index of the function with the maximum gap at $\bar{\x}$. Since $I(k) \in \argmax_{i\in[m]} f_i(\x_k) - f_i^\star$, it follows that for every $k$:
    \begin{align*}
        f_{I_e}(\x_k) - f_{I_e}^\star \leq f_{I(k)}(\x_k) - f_{I(k)}^\star.
    \end{align*}
    By convexity of $\{f_i\}_{i \in [m]}$ and Jensen's inequality, we have $f_{I_e}(\bar{\x}) \leq \frac{1}{K} \sum_{k=1}^K f_{I_e}(\x_k)$. Summing over $k$:
    \begin{align*}
        \max_{i\in[m]} f_i(\bar{\x}) - f_i^\star = f_{I_e}(\bar{\x}) - f_{I_e}^\star \leq \frac{1}{K} \sum_{k=1}^K \brac{f_{I_e}(\x_k) - f_{I_e}^\star} \leq \frac{1}{K} \sum_{k=1}^K \brac{f_{I(k)}(\x_k) - f_{I(k)}^\star}.
    \end{align*}
\end{proof}

This result implies that we can minimize the MG of the average iterate by generating a sequence of iterates with low average regret with respect to the functions selected by \MGAMOO, which always chooses the function with the largest current gap. 

A direct corollary of Lemma~\ref{lemma:worst_case_gap_avg} is a reduction from AMOO to online learning~\citep{cesa2006prediction,hazan2016introduction,orabona2019modern}. The online learning literature provides algorithms—such as online GD, online mirror descent, and online Newton step—that achieve sublinear regret on sequences of adaptive adversarially chosen functions \citep{hazan2016introduction}:
\begin{align}
    \mathrm{Regret} (K):= \sup_{\{f_1,\ldots,f_K\} \subseteq \mathcal{F}} \brac{ \sum_{k=1}^K f_k(\x_k) -  \sum_{k=1}^K f_t(\x)} \in o(K), \label{eq:online learning guarantee}
\end{align}
for any bounded convex functions, $\{f_k\}_{k=1}^K$, over $\mathcal{X}$, and for any $\x \in \mathcal{X}$, when $\mathcal{X}$ is a bounded convex set. 


Instantiating \MGAMOO\ with such an online learner yields the following guarantee:

\begin{corollary}[Online to Aligned Multi-Objective (O2AMO)]\label{col:online_reduction}
    Let $\mathcal{C}_\star$ a bounded set. Let $\{f_{i}\}_{i\in[m]}$ satisfy the AMOO assumption and be bounded over $\mathcal{X} := Convex Hull \big{\{}\{\x_1, \dots, \x_K \} \bigcup \mathcal{C}_\star \big{\}}$. Instantiating the single-objective optimizer \texttt{SOO} in \MGAMOO\ with an adaptive adversary online learning algorithm with regret $\mathrm{Regret}(K)$, guarantees that $MG(\bar{\x}) \leq \mathrm{Regret}(K) / K$.
\end{corollary}


\begin{proof}[Proof of Corollary \ref{col:online_reduction}]
    Recall that $\{f_{I(k)}\}_{k=1}^K \subseteq \{f_i\}_{i=1}^m$ are bounded and convex over $\mathcal{X}$, which is bounded and convex itself. Then, by the regret guarantee in \eqref{eq:online learning guarantee}, it holds for any $\x \in \mathcal{X}$ that $\sum_{k=1}^K f_{I(k)}(\x_k) - f_{I(k)}(\x) \leq \mathrm{Regret}(K) \in o(K)$, where $I(k)$ is chosen adversarially based on $x_k$ (see Algorithm~\ref{alg:MGAMOO-Reduction}). Setting $\x = \x_\star \in \mathcal{C}_\star \subseteq \mathcal{X}$ and applying Lemma~\ref{lemma:worst_case_gap_avg} completes the proof.
\end{proof}

This guarantee is analgous to the well-known Online-to-Batch (O2B) conversion. In words, if an online learner achieves sublinear regret, then averaging its iterates over a data sequence yields strong convergence guarantees in the \MGAMOO~setting.


We also analyze two additional and widely used single-objective optimization algorithms—GD and the Polyak step-size design method—and show that both can be applied in the smooth and Lipchitz settings. The Polyak step-size design and GD in the smooth case require different specialized analysis to bound the term in the right hand side of \eqref{eq:lemma_2_main_lemma}. The analysis of these algorithms and settings is not included in standard online learning analysis. The result for GD in the Lipchitz case is a direct application of standard online learning analysis. 




\begin{restatable}[\MGAMOO: Gradient-Descent and Polyak step size method]{theorem}{MGPAmooConvergence}\label{thm:MG_AMOO} Assume that \emph{Polyak} or \emph{GD} methods are used as the \texttt{SOO}. \MGAMOO~(Algorithm~\ref{alg:MGAMOO-Reduction}) guarantees that after $K$ iterations the worst-case gap is upper bounded by:
    \begin{align*}
        (1) \quad & \MG(\bar{\x}) := \max_{i\in[m]} f_i(\bar{\x}) - f_i^\star \leq \frac{3}{2}\frac{ G \enorm{\x_1 - \x_\star} }{\sqrt{K}} , \quad \text{when $\{ f_i(\x) \}_{i\in [m]}$ are $G$-Lipschitz.} \\
        (2) \quad & \MG(\bar{\x}) := \max_{i\in[m]} f_i(\bar{\x}) - f_i^\star \leq \frac{2\beta \enorm{\x_1 - \x_\star} }{K}, \quad \text{ when $\{ f_i(\x) \}_{i\in [m]}$ are $\beta$-smooth.}
    \end{align*}
\end{restatable}
The proof of the theorem is in Appendix~\ref{app:MG_AMOO}.

\subsection{Computational Complexity}\label{sec:computational}



Among the three algorithmic approaches considered in this work, \PAMOO\ requires the highest computational cost. Specifically, computing the Jacobian incurs a cost of $O(nm^2)$ operations, where $n$ is the dimension of the parameter space and $m$ is the number of objective functions. Additionally, each iteration requires solving a quadratic constrained optimization problem. Although certain techniques may reduce the overall burden, the need to compute gradients with respect to all functions and solve a convex optimization problem at each step introduces significant computational and implementation complexity.

In contrast, both the \EW\ and \MGAMOO\ methods require only $O(m)$ operations per iteration: the former to sum the functions, and the latter to identify the function with the maximum gap. The overall computational cost of these methods also depends on the complexity of the single-objective optimizer (\texttt{SOO}) employed. Since both \EW\ and \MGAMOO\ can leverage simple optimizers such as GD or projected GD, they are often substantially cheaper to apply in practice.

\section{Relaxing AMOO Assumption: $\epsilon$-Approximate AMOO}

The alignment assumption requires that all objective functions share a common optimal solution. However, in practice, we may only expect the existence of a near-optimal solution that approximately minimizes all functions. To address such scenarios, we seek algorithms that are robust to this relaxation of the alignment setting.

To formalize this approximate scenario, we define the set of $\epsilon$-approximate solutions:
\begin{align}
    \mC_\epsilon = \left\{ \x \in \reals^n \mid f_i(\x) - f_i^\star \leq \epsilon \quad \forall i \in [m] \right\}. \label{eq:Ce_set}
\end{align}
We assume that $\mC_\epsilon$ is non-empty and contains at least one element. This assumption generalizes the alignment condition in AMOO; indeed, setting $\epsilon = 0$ recovers the original AMOO setting.

The analytical tools developed in this work extend naturally to the $\epsilon$-approximate AMOO setting. In particular, we show that all the proposed algorithms remain applicable, recovering solutions that are $\epsilon$-near-optimal across all objectives. Below we provide an informal statement of the result, and refer the reader to the appendix for complete proofs.

\begin{restatable}[(Informal) Convergence in Convex $\epsilon$-Approximate AMOO]{theorem}{ConvexAppAMOO}
\label{thm:convex_app_convex_amoo} 
Let $\epsilon > 0$ and assume $\mC_\epsilon$ is non-empty. Then, \PAMOO\ and \MGAMOO, instantiated with an online learner, GD, or the Polyak step-size method, output a solution $\bar{\x}$ satisfying asymptotically $\MG(\bar{\x}) \leq \epsilon$. 
\end{restatable}

This result implies that both \PAMOO\ and \MGAMOO\ are applicable in the $\epsilon$-approximate AMOO setting. Notably, this substantially improves upon the results of~\citet{efroni2025aligned}: their results are applicable in the approximate case only if $\epsilon \leq \epsilon_0$, with $\epsilon_0$ depending on structural parameters such as smoothness and strong convexity. In contrast, our analysis removes this restriction and establishes that both \PAMOO\ and \MGAMOO\ can recover an $\epsilon$-near-optimal solution for any value of $\epsilon$.

\section{Experiments} \label{sec:experiments}

In this section, we tested the algorithms \EW, \PAMOO\ and \MGAMOO\ the reduction approach we introduced in this work. We experimented with an additional variation of \MGAMOO\ with weight averaging or momentum, as we refer to it. That is, we calculate the new weight vector $\w_{\mathrm{new}}$ according to the \MGAMOO\ update rule and update the weight by $\w_{k+1}=(1-\beta_{\mathrm{mom}}) \w_k + \beta_{\mathrm{mom}} \w_{\mathrm{new}}.$ In all experiments, we set $\beta_{\mathrm{mom}}=0.95$ similar to the default setting of other optimizers.

\begin{table}[t]
  \centering
  \begin{tabular}{lcc}
    \toprule
    \textbf{Algorithm} & \textbf{SGD [iter/time]} $\uparrow$  & \textbf{Adam [iter/time]} $\uparrow$ \\
    \midrule
    \EW & 1 & 1 \\
    \PAMOO & 0.07 & 0.08 \\
    \MGAMOO\ w/ Polyak & 0.97 & 0.96 \\
    \MGAMOO\ w/ GD & 0.84 & 0.85 \\
    \bottomrule
  \end{tabular}
  \vspace{0.1cm}
  \caption{Average iterations per time across problems (P1)-(P3) relative to \EW. \MGAMOO\ and \EW\ running time are substantially better compared to \PAMOO.}
  \label{tab:running_time}
\end{table}

Similar to \citet{efroni2025aligned}, in the learning problem we address, one network is tasked with replicating the outputs of another fixed network. The fixed network, characterized by parameters $\theta_\star$, is represented as $h_{\theta_\star}: \mathbb{R}^{d_i} \rightarrow \mathbb{R}^{d_o}$. The second network, with parameters $\theta$, is denoted as $h_{\theta}: \mathbb{R}^{d_i} \rightarrow \mathbb{R}^{d_o}$. Both networks are 2-layer neural networks utilizing ReLU activation functions and consist of 512 hidden units.

We draw data from a uniform distribution $\mathcal{D}=\{ \x_i \}_{i}$ where $\x_i\in \mathrm{Uniform}([-1,1]^{d_i})$. We consider three loss functions:
\begin{align*}
    &\forall i \in [3]: \quad f_i(\theta) \!=\! \frac{1}{|\mathcal{D}|}\!\!\sum_{\x\in \mathcal{D}}\!\!\brac{(h_\theta(\x)\! -\! h_{\theta_\star}(\x)-\epsilon_i)^\top \H_i (h_\theta(\x)\! -\! h_{\theta_\star}(\x)) -\epsilon_i}^{\alpha_i}\!,
\end{align*}
where $\H_i\in \mathbb{R}^{d_o\times d_o}$  is a positive definite matrix and $\alpha_i\geq 1$.

We choose 3 optimization problems (P1)-(P3) that demonstrate the behavior and robustness of our method, problems become progressively harder.
Problem (P1) is the switching example of~\citet{efroni2025aligned}. We set $\H_i=\I$, and $\alpha_i\in \{1,1.5,2\}$ and $\epsilon_i=0$. This implies the alignment assumption is satisfied. For such a choice $f_1$ has larger curvature for large losses whereas $f_2$ and $f_3$ have larger curvature  for small losses. Hence, we expect the the multi-objective feedback can allow the optimizer to focus on different functions in different parts of the feature space. We demonstrate the switching of weights empirically in \cref{app:weights}.  
Problems (P2) is similar to (P1), while setting $\epsilon_i\in \{0,0.05,-0.05\}$, which tests the setting in which the functions are not perfectly aligned.
Problem (P3) additionally tests the performance of the algorithms for scenario $f_1,f_2$ and $f_3$ are not strongly convex in the network outputs, namely, when the Hessian is ill-conditioned. We implement this by setting the problem dimension $d_i=100$, the Hessian to $\H_i=0.5\cdot\mathrm{Diag}(u_1+1,\ldots,(u_{90}+1)10^{-3}, \ldots, (u_{100}+1)10^{-3})$, with $\epsilon_i\in \{0,0.01,-0.01\}$. Our experiment is inspired by problem P1 tested in~\citep{schneider2019deepobs,chaudhari2019entropy}.

We used \EW, \PAMOO\ and \MGAMOO\ as the algorithms that choose the weight vector at each iteration $\w_k$, and applied SGD or ADAM optimizers to the weighted loss by $\w_k$. Figure~\ref{fig:main} depicts the convergence of the algorithms. As observed, \PAMOO, which requires more compute compared to \MGAMOO\ and \EW\ achieves the best performance, and \MGAMOO\ is substantially better than \EW, while suffering with no degradation in its running time (see Table~\ref{tab:running_time}). Additionally, it is evident the algorithms are also robust to small violations in the alignment assumption: the algorithms consistently converge on problems (P2) for instances of the $\epsilon$-approximate AMOO setting.



\begin{figure}[t]
  \centering
  \begin{tabular}{ccc}
    \begin{subfigure}{0.3\textwidth}
      \centering
      \includegraphics[width=\linewidth]{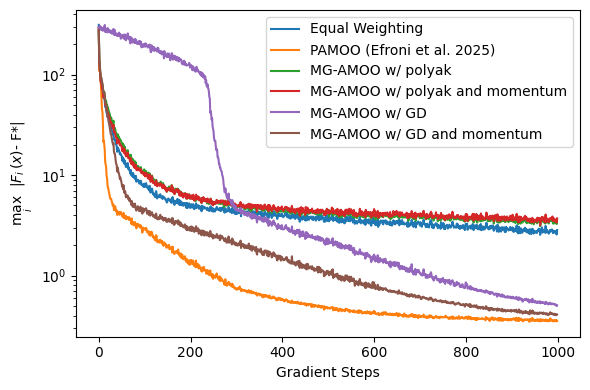}
      \caption{(P1) w/ SGD backend}
      \label{fig:sub1}
    \end{subfigure} &
    \begin{subfigure}{0.3\textwidth}
      \centering
      \includegraphics[width=\linewidth]{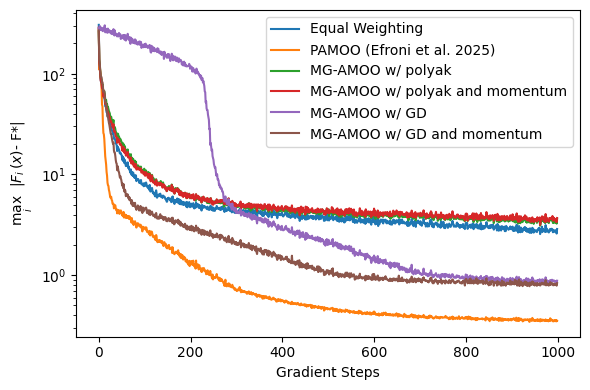}
      \caption{ (P2) w/ SGD backend}
      \label{fig:sub2}
    \end{subfigure} &
    \begin{subfigure}{0.3\textwidth}
      \centering
      \includegraphics[width=\linewidth]{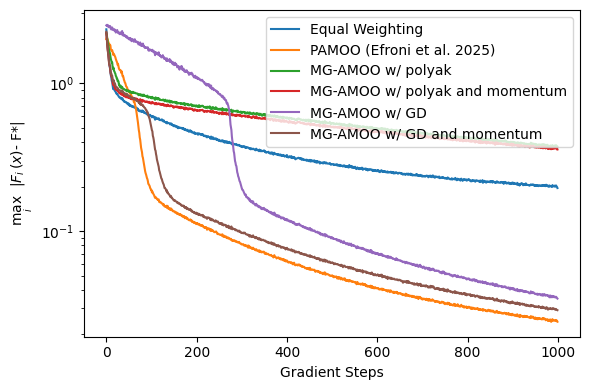}
      \caption{(P3) w/ SGD backend}
      \label{fig:sub3}
    \end{subfigure} \\
    \begin{subfigure}{0.3\textwidth}
      \centering
      \includegraphics[width=\linewidth]{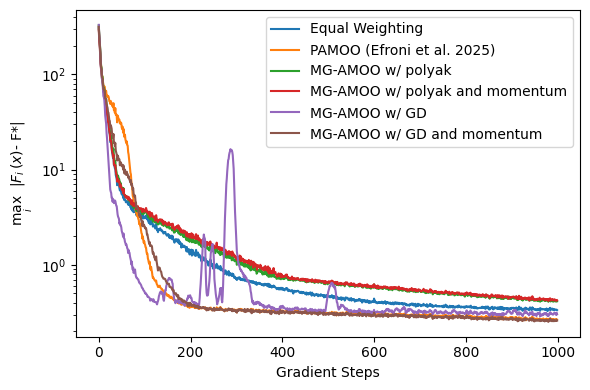}
      \caption{(P1) w/ Adam backend}
      \label{fig:sub4}
    \end{subfigure} &
    \begin{subfigure}{0.3\textwidth}
      \centering
      \includegraphics[width=\linewidth]{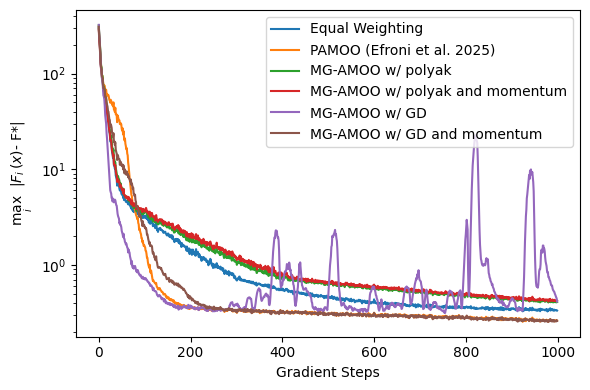}
      \caption{ (P2) w/ Adam backend}
      \label{fig:sub5}
    \end{subfigure} &
    \begin{subfigure}{0.3\textwidth}
      \centering
      \includegraphics[width=\linewidth]{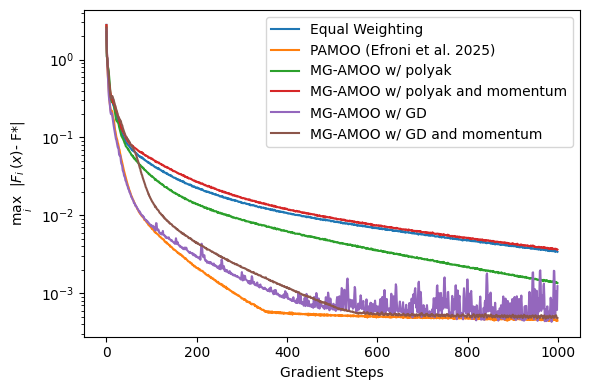}
      \caption{(P3) w/ Adam backend}
      \label{fig:sub6}
    \end{subfigure}
  \end{tabular}
  \caption{Performance of different AMOO algorithms with SGD (top) and Adam (bottom) backend. We show the convergence properties in three increasingly difficult problems: our switching problem (left), the switching problem when the optimum is not exactly alighned (middle), and when the problem is additionally ill-conditioned (right). When we augment our algorithm with momentum it can almost recover the performance of \texttt{PAMOO} at a fraction of the computational cost to run it. }
  \label{fig:main}
\end{figure}

\section{Conclusions, Limitations \& Future Work} 



We studied the convex AMOO setting, introducing the MG metric and new analytical tools that remove the need for strong convexity. We proposed \MGAMOO, a simple meta-algorithm that leverages standard single-objective optimizers, and established convergence guarantees for its variants, as well as for the more compute-intensive \PAMOO. We also derived a non-asymptotic lower bound for \EW, showing a suprising disadvantage of its: its performance may degrade as the number of objectives increases. Additionally, we showed that both \MGAMOO\ and \PAMOO\ remain robust even when the alignment assumption is relaxed. Experiments confirm that \MGAMOO\ achieves faster convergence than \EW\ without incurring the high computational cost of \PAMOO.

Future work can address the limitations of this study. Neither our work nor prior efforts have explored the stochastic or non-convex AMOOO settings, which are important for practical applications. Further research should also include large-scale experiments on computer vision, reinforcement learning, or LLMs with multiple reward signals. We conjecture that AMOO algorithms may benefit general multi-objective optimization by first identifying a strong common solution, followed by fine-tuning to satisfy specific trade-offs. This approach offers a new algorithmic perspective with potential impact on a wide range of machine learning applications in the future.

%% file: paper/appendix.tex
\section{Experiments}

In this section we provide additional details on the experiments. We also share a Python notebook with  implementations of \EW, \PAMOO, and variations of \MGAMOO\, the meta-algorithm we introduce in this work.

\paragraph{Dataset.}
We sample 1000 batches of 1000 points from an independent uniform distribution, $\x_i \sim \mathrm{Uniform}([-1,1]^{20})$. The target data is generated using a randomly initialized network defined as $t(x_i) = h_{\theta_\star}(x_i) + 10$. The output dimension of the target is $100$.

\paragraph{Network architecture.}
Both the ground truth and target networks share the same architecture: two-layer neural networks with 512 hidden units and ReLU activations. Each network outputs a vector in $\mathbb{R}^7$.

\paragraph{Training.}
For both tasks, we set the SGD learning rate to $0.0005$ and the ADAM learning rate to $0.002$, for \EW, \PAMOO\ and \MGAMOO\ variations of calculating the weight vector.

\paragraph{Implementation details \MGAMOO.}
We implemented the \MGAMOO\ variations by calculating the function with maximum gap in each iteration $\w_{\mathrm{new}}(k)\in \arg\max_{i\in [m]} f_i(\x_k) - f_i^\star$. When using momentum we set the weight vector on the $k^{\mathrm{th}}$ iteration as
\begin{align*}
    \w_k = (1-\beta_{\mathrm{mom}}) \w_{k-1} + \beta_{\mathrm{mom}} \w_{\mathrm{new}}(k), 
\end{align*}
where $\beta_{\mathrm{mom}}=0.95$. Otherwise, with momentum, we set
\begin{align*}
    \w_k = \w_{\mathrm{new}}(k).
\end{align*}
Finally, the loss at the $k^{\mathrm{th}}$ iteration is being set as $f_{\w_k}(\x)=\sum_{i\in[m]}w_i f_i(\x)$.

When using the Polyak variation of \MGAMOO\ we also scale the loss by the Polyak step-size $\eta_k =\frac{f_{\w_k}(\x_k) - f_{\w_k}^\star} {\enorm{\nabla f_{\w_k}(\x_k)}^2}$, namely, we apply an optimization step with either SGD or ADAM optimizers on $\eta_k f_{\w_k}(\x)$. For the GD variation of \MGAMOO\ we do not perform such additional scaling.


\paragraph{General parameters for \PAMOO.} We use the implementation of~\citet{efroni2025aligned} and use the projected gradient-descent based approach to update the weight vector.

\paragraph{Computational Resources.} Our experiments were run on CPUs and can be easily replicated. The average running time of the experiments is between 0.1-2 minutes. 

\paragraph{Additional Plots.}\label{app:weights} We include additional experimental results that depicts the iteration of the weight vector as a function of the iteration (see  Figure~\ref{fig:weights_main}). These show either a switching behavior--as the iterates converge the weight vector switches between different functions-- or selection behavior--where the algorithms focus on the more informative function. Further, it is evident that the momentum term that supports weight averaging improves the stability of the weight vector. 

\begin{figure}[t]
  \centering
  \begin{tabular}{ccc}
    \begin{subfigure}{0.28\textwidth}
      \centering
      \includegraphics[width=\linewidth]{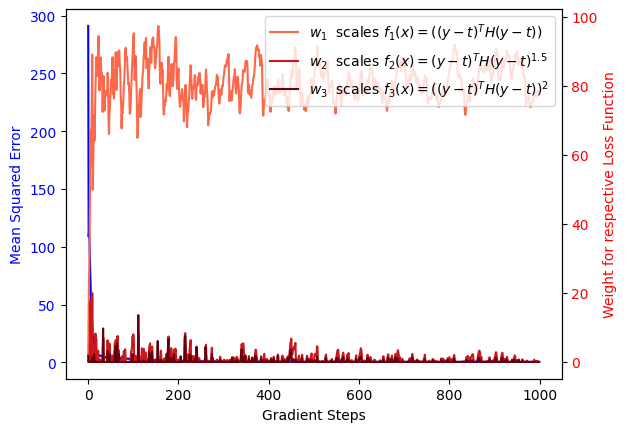}
      \caption{\texttt{PAMOO} SGD}
      \label{fig:sub1_weights}
    \end{subfigure} &
    \begin{subfigure}{0.28\textwidth}
      \centering
      \includegraphics[width=\linewidth]{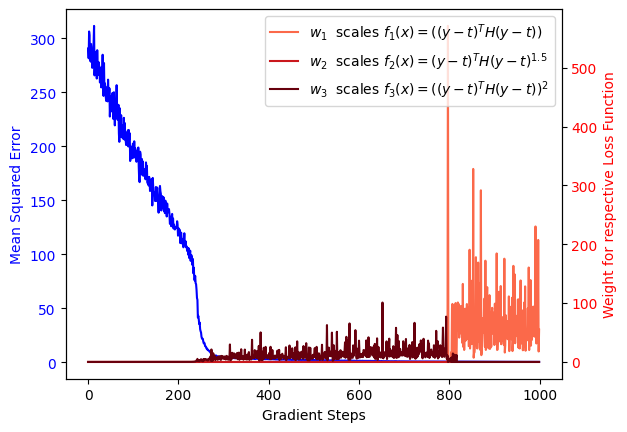}
      \caption{ \texttt{MG-AMOO} w/ Polyak, SGD}
      \label{fig:sub2_weights}
    \end{subfigure} &
    \begin{subfigure}{0.28\textwidth}
      \centering
      \includegraphics[width=\linewidth]{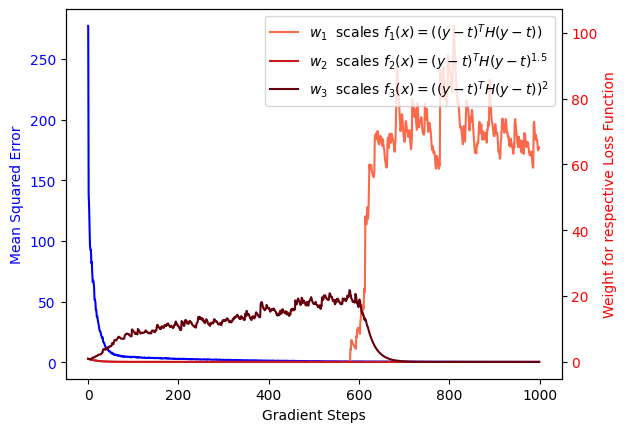}
      \caption{\texttt{MG-AMOO} w/ Polyak, momentum, SGD}
      \label{fig:sub3_weights}
    \end{subfigure} \\
    \begin{subfigure}{0.28\textwidth}
      \centering
      \includegraphics[width=\linewidth]{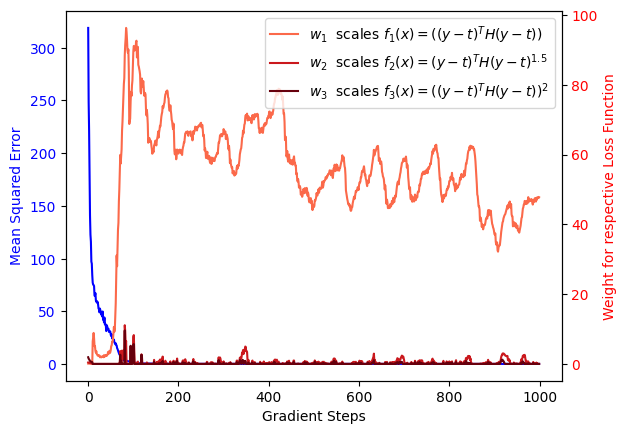}
      \caption{\texttt{PAMOO} Adam}
      \label{fig:sub4_weights}
    \end{subfigure} &
    \begin{subfigure}{0.28\textwidth}
      \centering
      \includegraphics[width=\linewidth]{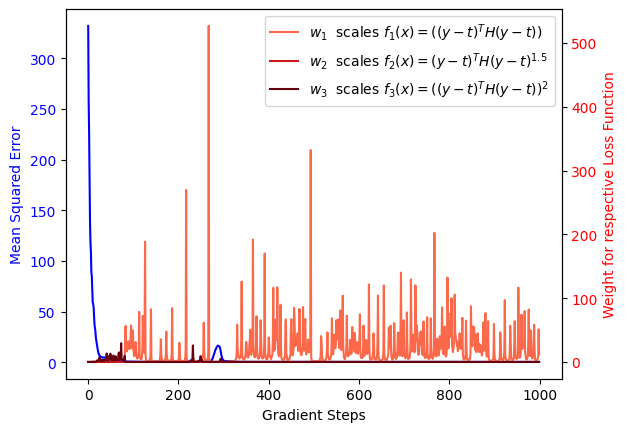}
      \caption{ \texttt{MG-AMOO} w/ Polyak, ADAM}
      \label{fig:sub5_weights}
    \end{subfigure} &
    \begin{subfigure}{0.28\textwidth}
      \centering
      \includegraphics[width=\linewidth]{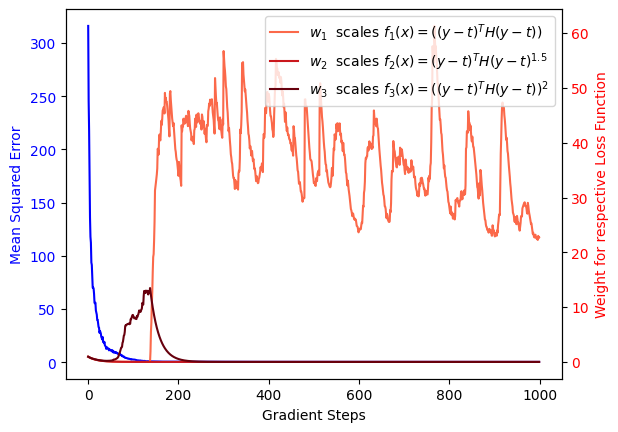}
      \caption{\texttt{MG-AMOO} w/ Polyak, Momentum, ADAM}
      \label{fig:sub6_weights}
    \end{subfigure}
  \end{tabular}
  \caption{The behavior of the three weights (reds) in (P1) shown as a function of the loss (blue) and gradients steps. Once the loss is below 1, the weights shift from $w_3$ to $w_1$ for all three algorithms. When no momentum is applied there seems to be a complex interaction between ADAM and \PAMOO\ that leads to overall instabilities. }
  \label{fig:weights_main}
\end{figure}

\newpage

\section{Proof of Claims}

\subsection{\CAMOO\ is Invalid in Convex AMOO}\label{app:camoo_is_invalid}

Consider the \CAMOO\ algorithm that iteratively calculates the weights by $\w_k\in \arg\max_{\w\in \Delta_m}\lambda_{\min}\brac{\nabla^2 f_{\w}(\x_k)}$. Further, consider the multi-objective problem with the following two functions in $\x=(x_1,x_2,x_3)\in \mathbb{R}^3$:
\begin{align*}
    f_1(\x) =x_1^2,\\
    f_2(\x) = x_2^2.
\end{align*}
Both functions are only convex and share optimal solutions in the set $\mC_\star = \{ \x: \x=0 \}$.  

In this problem it holds that $\max_{\w\in \Delta_2}\lambda_{\min}\brac{\nabla^2 f_{\w}(\x_k)}=0$ for all $\w\in \mathbb{R}^2$ since 
\begin{align*}
\nabla^2 f_{\w}(\x) = w_1 x_1^2 + w_2x^2 =   \mathrm{Diag}(2w_1,2w_2,0).  
\end{align*}
Hence, for any value $\w\in \Delta_3$ it holds that $\max_{\w\in \Delta_2}\lambda_{\min}\brac{\nabla^2 f_{\w}(\x_k)}=0$, and hence, it is possible that $\arg\max_{\w\in \Delta_m}\lambda_{\min}\brac{\nabla^2 f_{\w}(\x_k)}=\brac{0,0,1} $ which is makes no change in $\x$. for any value $\w\in \Delta_2$ it holds that $\max_{\w\in \Delta_2}\lambda_{\min}\brac{\nabla^2 f_{\w}(\x_k)}=0$, and hence, $\arg\max_{\w\in \Delta_m}\lambda_{\min}\brac{\nabla^2 f_{\w}(\x_k)}=\Delta_2$ is the entire simplex. This implies that \CAMOO, without any modification, will produce arbitrary weight vectors.

\subsection{$\x_{\EW}$ Minimizes the Maximum Gap Metric}\label{app:x_ew_minimizes_mg}

Restate the claim:

\ClaimOne*

We will prove by contradiction. 

Let $\x_{\MG}\in \mC_\star$ be a point that minimizes all objective functions and, hence, satisfies $\MG(\x_{\MG})=0$.  Such a point necessarily exists by AMOO assumption. Assume that  $\MG(\x_{\EW})>0$. This will be shown to lead to a contradiction. We will show that $f_{\EW}(\x_{\MG})< f_{\EW}(\x_{\EW})$, contradicting the optimality of $\x_{\EW}$ for $f_{\EW}$.

Since $\MG(\x_{\EW})>0$ it implies that exists $j\in [m]$ such that
\begin{align*}
    f_j(\x_{\MG}) =  \min_{\x \in \reals^n} ~ f_j(\x)  <   f_j(\x_{\EW}). 
\end{align*}
Hence,
\begin{align*}
    f_{\EW}(\x_{\MG}) =  \frac{1}{m} \sum_{i\in [m]} f_i(\x_{\MG}) = \frac{1}{m} \sum_{i\in [m]} \min_{\x \in \reals^n} ~ f_i(\x) < \frac{1}{m} \sum_{i\in [m]} f_i(\x_{\EW}) = f_{\EW}(\x_{\EW}). 
\end{align*}
This contradicts the optimality of $\x_{\EW}$ on the objective $f_{\EW}$.

\section{Auxiliary Lemmas: \EW\ Lower Bound, Theorem~\ref{thm:lower_bound_ew_polyak}}
\subsection{Induction Lower Bound} \label{app:induction_lower_bound}

In this section we present auxiliary results for Theorem~\ref{thm:lower_bound_ew_polyak}.
\begin{lemma}\label{lemma:induction_lower_bound}
    Let $a, \epsilon > 0 $, $n \geq m \geq 2$, and $a=(m-1)\epsilon$. Suppose the sequence $\{\x_1, \x_2, \dots \}$, when $\x_k = (x_{k,1}, \dots, x_{k,n}) \in \reals^n$, defined by the following:
    \begin{align}
    \x_{k+1} = \x_k - \frac{\sum_{i=1}^{m} |x_{k,i}|}{m} \brac{\mathrm{sign}(x_{k,1}), \dots, \mathrm{sign}(x_{k,m}), 0, \dots, 0}. \label{eq:update_step_induction}
\end{align}
When $\x_1 = (a, \epsilon, \dots, \epsilon, 0, \dots, 0) \in \reals^n$. Then, the followings hold:
\begin{enumerate}
    \item 
    \begin{align*}
        x_{k,1} = a\brac{1-\frac{2}{m}}^{k-1} .
    \end{align*}
    \item for every $1<i\leq m$ \begin{align*}
        x_{k,i} = \epsilon (-1)^{k-1} \brac{1 - \frac{2}{m}}^{k-1}.
    \end{align*}
\end{enumerate}
\end{lemma}
\begin{proof}[Proof of Lemma \ref{lemma:induction_lower_bound}]
We prove it by induction.\\
\emph{Base case:} when $k=2$.\\
Using the update step from \eqref{eq:update_step_induction}, we have
\begin{align*}
    \x_{2} & = \x_1 - \frac{a + (m-1) \epsilon}{m} \brac{\mathrm{sign}(x_1), \dots, \mathrm{sign}(x_{m}), 0, \dots, 0} \\
    & = \x_1 - \frac{a + (m-1) \epsilon}{m} \brac{+1, \dots, +1, 0, \dots, 0}.
\end{align*}
Recall that $a = \epsilon(m-1)$. Then, for every $1<i\leq m$ we have
\begin{align*}
    x_{2,i} & = x_{1,i} - \frac{a + (m-1) \epsilon}{m}= \epsilon - \frac{2 (m-1) \epsilon}{m} = \epsilon \brac{ 1 - 2 + \frac{2}{m}} = \epsilon (-1) \brac{1-\frac{2}{m}},
\end{align*}
and for $i=1$ we have
\begin{align*}
    x_{2,1} & = x_{1,1} - \frac{a + (m-1) \epsilon}{m}= a - \frac{2 a}{m} = a \brac{ 1 -  \frac{2}{m}}.
\end{align*}
Thus, we obtain
\begin{align*}
    x_{2,1} = a\brac{1-\frac{2}{m}}, ~~~~ \text{and} ~~~~ x_{2,i} = \epsilon (-1) \brac{1- \frac{2}{m }} \text{ for every } 1<i\leq m.
\end{align*}
\emph{Induction hypothesis:}
Suppose that for $k=n$. \\
Then, it holds:
\begin{align*}
    x_{n,1} = a\brac{1-\frac{2}{m}}^{n-1}, ~~~~ \text{and} ~~~~ x_{n,i} = \epsilon (-1)^{n-1} \brac{1- \frac{2}{m }}^{n-1} \text{ for every } 1<i\leq m.
\end{align*}
\emph{Induction step:} we prove for $k=n+1$.\\
Using the update step from \eqref{eq:update_step_induction}, we have
\begin{align*}
    \x_{n+1} & = \x_n - \frac{\sum_{i=1}^{m} |x_{n,i}|}{m} \brac{\mathrm{sign}(x_{n,1}), \dots, \mathrm{sign}(x_{n,m}), 0, \dots, 0}\\
    & = \x_n - \frac{|a\brac{1-\frac{2}{m}}^{n-1}| + (m-1) | \epsilon (-1)^{n-1} \brac{1- \frac{2}{m }}^{n-1}|}{m} \brac{\mathrm{sign}(x_{n,1}), \dots, \mathrm{sign}(x_{n,m}), 0, \dots, 0}.
\end{align*}
From the \emph{induction hypothesis} it holds that $x_{n,1} = a\brac{1-\frac{2}{m}}^{n-1} > 0 $. Since $a=(m-1)\epsilon$, we have 
\begin{align*}
    \x_{n+1} & = \x_n - \frac{2 a\brac{1-\frac{2}{m}}^{n-1} }{m} \brac{+1, \mathrm{sign}(x_{n,2}), \dots, \mathrm{sign}(x_{n,m}), 0, \dots, 0}.
\end{align*}
Thus, we obtain
\begin{align*}
    x_{n+1,1} &=  x_{n,1} - \frac{2 a\brac{1-\frac{2}{m}}^{n-1}}{m}(+1) \\
    &= a\brac{1-\frac{2}{m}}^{n-1} - \frac{2 }{m} a\brac{1-\frac{2}{m}}^{n-1} \\
    &= a\brac{1-\frac{2}{m}}\brac{1-\frac{2}{m}}^{n-1}\\
    &=a\brac{1-\frac{2}{m}}^{n}
\end{align*}
Now we move to show that the lemma holds for $x_{n+1,i}$ for every $1<i\leq m$. We separate this into two cases.\\
\emph{Case 1:} $n$ is odd. \\
From the \emph{induction hypothesis} it holds that $x_{n,i} = \epsilon (-1)^{n-1} \brac{1- \frac{2}{m }}^{n-1}$. Thus, $x_{n,i} > 0$ for every $1<i\leq m$, and it holds that
\begin{align*}
    \x_{n+1} &  =  \x_n - \frac{2 a\brac{1-\frac{2}{m}}^{n-1}}{m} \brac{+1, +1, \dots, +1, 0, \dots, 0}.
\end{align*}
Recall that $a=\epsilon(m-1)$, and $n-1$ is even. Then, we obtain:
\begin{align*}
    x_{n+1,i} & =  x_{n,i} - \frac{2 a\brac{1-\frac{2}{m}}^{n-1}}{m}(+1) = \epsilon (-1)^{n-1} \brac{1- \frac{2}{m }}^{n-1} - 2\frac{ m-1}{m} \epsilon \brac{1-\frac{2}{m}}^{n-1}\\
    & = \epsilon \brac{1- \frac{2}{m }}^{n-1} - 2\brac{ 1 - \frac{1}{m}} \epsilon \brac{1-\frac{2}{m}}^{n-1} = \epsilon \brac{1- \frac{2}{m }}^{n-1} \brac{1-2+\frac{2}{m}}\\
    & = \epsilon \brac{1- \frac{2}{m }}^{n-1} (-1) \brac{1-\frac{2}{m }} = \epsilon (-1)^n \brac{1- \frac{2}{m }}^{n}.
\end{align*}
\emph{Case 2:} $n$ is even.\\ From the \emph{induction hypothesis} it holds that $x_{n,i} = \epsilon (-1)^{n-1} \brac{1- \frac{2}{m }}^{n-1}$. Thus, $x_{n,i} < 0$ for every $1<i\leq m$, and it holds that
\begin{align*}
    \x_{n+1} & =  \x_n - \frac{2 a\brac{1-\frac{2}{m}}^{n-1}}{m} \brac{+1, -1, \dots, -1, 0, \dots, 0}.
\end{align*}
Recall that $a=\epsilon(m-1)$, and $n-1$ is odd. Then, we obtain:
\begin{align*}
    x_{n+1,i} & =  x_{n,i} - \frac{2 a\brac{1-\frac{2}{m}}^{n-1}}{m}(-1) = \epsilon (-1)^{n-1} \brac{1- \frac{2}{m }}^{n-1} + 2\frac{ m-1}{m} \epsilon \brac{1-\frac{2}{m}}^{n-1}\\
    & = - \epsilon \brac{1- \frac{2}{m }}^{n-1} + 2\brac{ 1 - \frac{1}{m}} \epsilon \brac{1-\frac{2}{m}}^{n-1} = \epsilon \brac{1- \frac{2}{m }}^{n-1} \brac{-1+2-\frac{2}{m}}\\
    & = \epsilon \brac{1- \frac{2}{m }}^{n-1} \brac{1-\frac{2}{m }} = \epsilon (-1)^n \brac{1- \frac{2}{m }}^{n}.
\end{align*}
\end{proof}
\subsection{Additional Result}
\begin{lemma}
    Let $K>0$, and $m \geq 2$. Then, it holds that:
\begin{align}
    \brac{1 - \brac{1-\frac{2}{m}}^K} \geq \frac{2K}{m+2K}.\label{eq:add_res_upper_bound_for_lower_bound}
\end{align}
\end{lemma}
\begin{proof}
Since $1-x \leq e^{-x}$, and since $e^{-x} \leq \frac{1}{1+x}$ for every $x \geq 0$, it holds that
\begin{align*}
    \brac{1-\frac{2}{m}}^K \leq \brac{e^{-\frac{2}{m}}}^K = e^{-\frac{2K}{m}} \leq \frac{1}{1+\frac{2K}{m}} = \frac{m}{m+2K}.
\end{align*}
Then, we have
\begin{align*}
    1-\brac{1-\frac{2}{m}}^K \geq 1 - \frac{m}{m+2K} = \frac{2K}{m+2K}.
\end{align*}
\end{proof}

\section{Auxiliary Lemma for Polyak step}
In this section we prove an auxiliary lemma for Polyak step, which is used to prove theorems from Section~\ref{sec:amoo_algorithms}.

\begin{lemma}\label{lemma:upper_bound_over_trajectory_gap}
Let $I(k) \in \argmax_{i\in[m]} f_i(\x_{k}) - f_i^\star$, and the trajectory $\{ \x_k \}_{k=1}^K$. Suppose for every $k$ the following inequality holds 

\begin{align}
    \enorm{\x_{k+1} - \x_\star}^2 \leq \enorm{\x_k - \x_\star}^2 - \frac{\brac{ f_{I(k)}(\x_k) - f_{I(k)}^\star}^2}{\enorm{\nabla f_{I(k)}(\x_k)}^2},
    \label{eq:conv_GDPolyakAMOO}
\end{align}
when $\x_\star \in \mC_\star$.
Then, Any algorithm which returns the trajectory average, $\bar{\x} = \frac{1}{K} \sum_{k=1}^K \x_k$, guarantees the worst-case gap is upper bounded by:
\begin{align*}
    (1) \quad & \max_{i\in[m]} f_i(\bar{\x}) - f_i^\star \leq \frac{ G \enorm{\x_1 - \x_\star} }{\sqrt{K}} , \quad \text{ when $\{ f_i(\x) \}_{i\in [m]}$ are $G$-Lipschitz.} \\
    (2) \quad & \max_{i\in[m]} f_i(\bar{\x}) - f_i^\star \leq \frac{2\beta \enorm{\x_1 - \x_\star} }{K}, \quad \text{when $\{ f_i(\x) \}_{i\in [m]}$ are $\beta$-smooth.}
\end{align*}
\end{lemma}
\begin{proof}[Proof of Lemma~\ref{lemma:upper_bound_over_trajectory_gap}]
    We separate to two cases.\\
    \emph{Case 1:}
    Since the functions are $G$-Lipschitz, i.e. $\enorm{\nabla f_{i}(\x) } \leq G $ for all $i \in [m]$, and $ \x \in \reals^n$. Then, from \eqref{eq:conv_GDPolyakAMOO}, we have
    \begin{align*}
        \brac{f_{I(k)}(\x_k) - f_{I(k)}^\star}^2 \leq  G^2 \brac{\enorm{\x_k - \x_\star}^2 - \enorm{\x_{k+1} - \x_\star}^2} .
    \end{align*}
    Summing over $k$, and using Cauchy–Schwarz inequality, i.e. $\brac{\sum_{i=1}^n a_i}^2 \leq n \sum_{i=1}^n a_i^2$ when $a_i \in \reals$, it holds
    \begin{align*}
        \brac{\sum_{k=1}^K \brac{f_{I(k)}(\x_k) - f_{I(k)}^\star}}^2 & \leq K \sum_{k=1}^K \brac{f_{I(k)}(\x_k) - f_{I(k)}^\star}^2 \leq   G^2 K \sum_{k=1}^K \brac{\enorm{\x_k - \x_\star}^2 - \enorm{\x_{k+1} - \x_\star}^2} \\
        & =  G^2 \enorm{\x_1 - \x_\star}^2 K.
    \end{align*} 
    Then by taking square root on both sides, we have 
    \begin{align}\label{eq:Lipchitz_general_max_gap}
        \sum_{k=1}^K \brac{f_{I(k)}(\x_k) - f_{I(k)}^\star} \leq  G \enorm{\x_1 - \x_\star} \sqrt{K} .
    \end{align}
    \emph{Case 2:} Since the functions are $\beta$-smooth, then, From Lemma~\ref{lemma:smooth-gradient-norm}, and \eqref{eq:conv_GDPolyakAMOO} it holds that
    \begin{align*}
        f_{I(k)}(\x_k) - f_{I(k)}^\star \leq  2\beta \brac{\enorm{\x_k - \x_\star}^2 - \enorm{\x_{k+1} - \x_\star}^2} .
    \end{align*}
    Summing over $k$, we have
    \begin{align} \label{eq:smooth_general_max_gap}
        \sum_{k=1}^K f_{I(k)}(\x_k) - f_{I(k)}^\star \leq 2\beta \enorm{\x_1 - \x_\star}^2.
    \end{align}
    By applying Lemma~\ref{lemma:worst_case_gap_avg} over \eqref{eq:Lipchitz_general_max_gap}, and ~\eqref{eq:smooth_general_max_gap} we obtain the proof.
\end{proof}

\section{Proof of Theorem~\ref{thm:original_pamoo_convergence}}\label{app:original_pamoo}
In this section we prove Theorem~\ref{thm:original_pamoo_convergence}. First, we prove first an auxiliary lemma, and then we prove the theorem.
\begin{lemma}\label{lemma:conv_PAMOOv1}
    Let $I(k) \in \argmax_{i\in[m]} f_i(\x_{k}) - f_i^\star$, and $\x_\star \in \mC_\star$. At each iteration $k$ in Algorithm \ref{alg:PAMOO_original} the following holds
    \begin{align*}
        \enorm{\x_{k+1} - \x_\star}^2 \leq \enorm{\x_k - \x_\star}^2 - \frac{\brac{ f_{I(k)}(\x_k) - f_{I(k)}^\star}^2}{\enorm{\nabla f_{I(k)}(\x_k)}^2} .
    \end{align*}
\end{lemma}
\begin{proof}[Proof of Lemma~\ref{lemma:conv_PAMOOv1}]
    Define that $f_{\w}(\x)$ is an weighted function of the functions according to $\w$, i.e. $f_{\w}(\x) = \sum_i w_i f_i(\x)$. Denote that for every $\w\in \reals^m_+$ it holds that $f_{\w}$ is a convex function. Hence, for every $\x,\y \in \reals^n$ it holds that
\begin{align*}
    &f_{\w}(\x) - f_{\w}(\y)  \leq \nabla f_\w(\x)^\top(\x-\y). 
\end{align*}
Let $\x_\star \in \mC_\star$. Recall that the update rule is given by $\x_{k+1} = \x_k - \nabla f_{\w_k}(\x_k)$. Then, by using the previous equation, we have
\begin{align}
    \enorm{\x_{k+1}-\x_\star}^2 &=\enorm{\x_k-\x_\star}^2 - 2 \nabla f_{\w_k}(\x_k)^\top\brac{\x_k - \x_\star} + \enorm{\nabla f_{\w_k}(\x_k)}^2 \nonumber\\
    &\leq \enorm{\x_k-\x_\star}^2 - 2 \brac{f_{\w_k}(\x_k)- f_{\w_k} (\x_\star)} +\enorm{\nabla f_{\w_k}(\x_k)}^2 , \label{eq:PAMOO_dist_from_x}
\end{align}
Recall that the update rule of \texttt{PAMOO} is 
$$
\w_k\in \argmax_{\w\in \reals^m_+} 2\brac{f_{\w}(\x_k)- f_{\w} (\x_\star)} - \enorm{\nabla f_{\w}(\x_k)}^2.
$$
Note that $I(k) \in \argmax_{i\in[m]} f_i(\x_{k}) - f_i^\star$ when $I(k)$ is a one-hot vector, $a_k = f_{I(k)}(\x_k)- f_{I(k)}^\star$, and $b_k = \enorm{\nabla f_{I(k)}(\x_k)}^2$. Let $\w(\x_k) = I(k) \frac{a_k}{b_k} \in \reals^m_+$, we can lower bound of the last expression as follows
\begin{align*}
    2 \brac{f_{\w_k}(\x_k)- f_{\w_k} (\x_\star)} - \enorm{\nabla f_{\w_k}(\x_k)}^2 & = \max_{\w\in \reals_+^m} \left[ 2\brac{f_{\w}(\x_k)- f_{\w} (\x_\star)} - \enorm{\nabla f_{\w}(\x_k)}^2 \right] \\
    & \geq 2\brac{f_{\w(\x_k)}(\x_k)- f_{\w(\x_k)} (\x_\star)}  -\enorm{\nabla f_{\w(\x_k)}(\x_k)}^2 \\
    & =   \frac{\brac{f_{I(k)}(\x_k)- f_{I(k)}^\star}^2}{\enorm{\nabla f_{I(k)}(\x_k)}^2}.
\end{align*}
Combining the last equation with \eqref{eq:PAMOO_dist_from_x}, we obtain
\begin{align*}
    \enorm{\x_{k+1}-\x_\star}^2 &\leq \enorm{\x_k-\x_\star}^2 -  \frac{\brac{f_{I(k)}(\x_k)- f_{I(k)}^\star}^2}{\enorm{\nabla f_{I(k)}(\x_k)}^2}
\end{align*}
\end{proof}

Restate the theorem:
\OriginalPAmooConvergence*

\begin{proof}[Proof of Theorem~\ref{thm:original_pamoo_convergence}]
    The theorem holds directly from Lemma~\ref{lemma:upper_bound_over_trajectory_gap}, which holds directly from Lemma~\ref{lemma:conv_PAMOOv1}. 
\end{proof}

\section{Proof of  Theorem~\ref{thm:MG_AMOO}}\label{app:MG_AMOO}

In this section we prove the following theorem: 
\MGPAmooConvergence*
\begin{proof}[Proof of Theorem~\ref{thm:MG_AMOO}]
    We provide the result for \PAMOO\ in Theorem~\ref{thm:MG_AMOO_Polyak}, the result for \MGAMOO\ with GD in Theorem~\ref{thm:MG_AMOO_GD_smooth}, and the result for \MGAMOO\ with Polyak step-size design in Theorem~\ref{thm:MG_AMOO_GD_Lip}. 

\end{proof}

\subsection{The Polyak method}\label{app:MG_AMOO_Polyak}

In this section we will prove the Polyak part of Theorem~\ref{thm:MG_AMOO}, namely, run \MGAMOO\ with the Polyak method as a \texttt{SOO}.
First, recall the update step in iteration $k$ of the Polyak method, which has the following form:
\begin{align}
    \label{eq:polyak_as_soo}
    \text{Set} \quad \eta_k = \frac{f(\x_{k}) - f^\star}{\enorm{\nabla f(\x_k)}^2}, \quad \text{define} \quad \g_k \gets \eta_k \nabla f(\x_k), \quad \text{and return} \quad \x_{k+1} = \x_k - \g_k.
\end{align}

Now, before we prove the Polyak part of Theorem~\ref{thm:MG_AMOO} we prove an auxiliary lemma for it.
\begin{lemma}\label{lemma:conv_MG_AMOO_Polyak_step}
    Let $\x_\star \in \mC_\star$. The Polyak update step, \eqref{eq:polyak_as_soo}, has the following attribute for every $k$:
    \begin{align*}
        \enorm{\x_{k+1} - \x_\star}^2 \leq \enorm{\x_k - \x_\star}^2 - \frac{\brac{ f(\x_k) - f^\star}^2}{\enorm{\nabla f(\x_k)}^2} .
    \end{align*}
\end{lemma}
\begin{proof}[Proof of Lemma~\ref{lemma:conv_MG_AMOO_Polyak_step}] Since $f(\x_{k}) - f^\star \geq 0$ for every $k$, it holds that $\eta_k \geq 0$ for every $k$. Since $\x_{k+1} = \x_k -  \eta_k \nabla f(\x_k)$, it holds that
    \begin{align*}
        \enorm{\x_{k+1} - \x_\star}^2 &= \enorm{\x_k - \eta_k \nabla f(\x_k) - \x_\star}^2 \\
        & = \enorm{\x_k - \x_\star}^2 - 2 \eta_k \nabla f(\x_k)^\top (\x_k - \x_\star) + \eta_k^2 \enorm{ \nabla f(\x_k)}^2\\
        &\leq \enorm{\x_k - \x_\star}^2 - 2 \eta_k \brac{ f(\x_k) - f^\star} + \eta_k^2 \enorm{ \nabla f(\x_k)}^2.
    \end{align*}
    The inequality holds from the convexity attribute, i.e. $f(\x_k) - f^\star  \leq \nabla f(\x_k)^\top(\x_k-\x_\star)$. Thus, by applying $\eta_k$ we obtain the lemma.
\end{proof}

\begin{theorem}\label{thm:MG_AMOO_Polyak}
\MGAMOO~(Algorithm~\ref{alg:MGAMOO-Reduction}) with Polyak method as \texttt{SOO} guarantees that after $K$ iterations the worst-case gap is upper bounded by:
\begin{align*}
        (1) \quad & \max_{i\in[m]} f_i(\bar{\x}) - f_i^\star \leq \frac{ G \enorm{\x_1 - \x_\star} }{\sqrt{K}} , \quad \text{ when $\{ f_i(\x) \}_{i\in [m]}$ are $G$-Lipschitz.} \\
        (2) \quad & \max_{i\in[m]} f_i(\bar{\x}) - f_i^\star \leq \frac{2\beta \enorm{\x_1 - \x_\star} }{K}, \quad \text{when $\{ f_i(\x) \}_{i\in [m]}$ are $\beta$-smooth.}
    \end{align*}    
\end{theorem}
\begin{proof}[Proof of Theorem~\ref{thm:MG_AMOO_Polyak}]
    The theorem holds directly from Lemma~\ref{lemma:upper_bound_over_trajectory_gap}, which holds directly from Lemma~\ref{lemma:conv_MG_AMOO_Polyak_step}, and the fact that the Polyak method, \eqref{eq:polyak_as_soo}, receive from \MGAMOO\ the function $f_{I(k)}(\x)$ at each iteration $k$.
    
\end{proof}

\subsection{The GD method for the smooth case}\label{app:MG_AMOO_GD}

In this section we will prove the Gradient Descent (GD) part of Theorem~\ref{thm:MG_AMOO}, namely, run \MGAMOO\ with the GD method as a \texttt{SOO}, only for the $\beta$-smooth case.
First, recall the update step in iteration $k$ of the GD method, which has the following form:
\begin{align}
    \label{eq:GD_as_soo_smooth}
    \text{Set} \quad  0 < \eta \leq \frac{1}{\beta}, \quad \text{define} \quad \g_k \gets \eta \nabla f(\x_k), \quad \text{and return} \quad \x_{k+1} = \x_k - \g_k.
\end{align}

Now, before we prove the GD part of Theorem~\ref{thm:MG_AMOO} we prove an auxiliary lemma for it.
\begin{lemma}\label{lemma:conv_MG_AMOO_GD_smooth}
    Let $\x_\star \in \mC_\star$, and $\eta = \frac{1}{2\beta} $. The GD update step, \eqref{eq:GD_as_soo_smooth}, has the following attribute for every $k$:
    \begin{align*}
        f(\x_k) - f^\star \leq \frac{1}{2\beta} \brac{\enorm{\x_k - \x_\star}^2 - \enorm{\x_{k+1} - \x_\star}^2}.
    \end{align*}
\end{lemma}
\begin{proof}[Proof of Lemma~\ref{lemma:conv_MG_AMOO_GD_smooth}]Recall that $\x_{k+1} = \x_k -  \eta \nabla f(\x_k)$, and $\eta > 0$. Since $f(\x)$ is convex, it holds that $f(\x_k) - f^\star  \leq \nabla f(\x_k)^\top(\x_k-\x_\star)$. By combining this convexity attribute with Lemma~\ref{lemma:smooth-gradient-norm}, it holds that
    \begin{align*}
        \enorm{\x_{k+1} - \x_\star}^2 &= \enorm{\x_k - \eta \nabla f(\x_k) - \x_\star}^2 \\
        & = \enorm{\x_k - \x_\star}^2 - 2 \eta \nabla f(\x_k)^\top (\x_k - \x_\star) + \eta^2 \enorm{ \nabla f(\x_k)}^2\\
        &\leq \enorm{\x_k - \x_\star}^2 - 2 \eta \brac{ f(\x_k) - f^\star} + 2 \beta \eta^2 \brac{ f(\x_k) - f^\star}\\
        &\leq \enorm{\x_k - \x_\star}^2 - 2 \eta \brac{1-\beta \eta}\brac{ f(\x_k) - f^\star}.
    \end{align*}
    We obtain the lemma by plugging-in $\eta$, and rearranging the equation.
\end{proof}

\begin{theorem}\label{thm:MG_AMOO_GD_smooth}
Suppose that $\{ f_i(\x) \}_{i\in [m]}$ are $\beta$-smooth. \MGAMOO\ with GD method with $\eta = \frac{1}{2\beta}$ as \texttt{SOO} guarantees that after $K$ iterations the worst-case gap is upper bounded by:
\begin{align*}
    \max_{i\in[m]} f_i(\bar{\x}) - f_i^\star \leq \frac{2\beta \enorm{\x_1 - \x_\star} }{K}.
\end{align*}    
\end{theorem}
\begin{proof}[Proof of Theorem~\ref{thm:MG_AMOO_GD_smooth}]
Let $I(k) \in \argmax_{i\in[m]} f_i(\x_{k}) - f_i^\star$, and the trajectory $\{ \x_k \}_{k=1}^K$. From Lemma~\ref{lemma:conv_MG_AMOO_GD_smooth} it holds that for every $k$ we have 
\begin{align*}
    f_{I(k)}(\x_k) - f_{I(k)}^\star \leq \frac{1}{2\beta} \brac{\enorm{\x_k - \x_\star}^2 - \enorm{\x_{k+1} - \x_\star}^2}.
\end{align*}
    Summing over $k$, we have that
    \begin{align*} 
        \sum_{k=1}^K f_{I(k)}(\x_k) - f_{I(k)}^\star \leq 2\beta \enorm{\x_1 - \x_\star}^2.
    \end{align*}
    By applying Lemma~\ref{lemma:worst_case_gap_avg} we obtain the proof.
\end{proof}

\subsection{The GD method for the $G$-Lipschitz case}\label{app:MG_AMOO_GD_lip}

\begin{theorem}[Theorem 3.1 in \cite{hazan2016introduction}]\label{thm:ogd}
Let $\eta_k = \frac{\enorm{\x_1 - \x_\star}}{G\sqrt{k}}$. Online Gradient Descent guarantees that:
\begin{align*}
    \mathrm{Regret}(K) \leq \frac{3}{2} G \enorm{\x_1 - \x_\star} \sqrt{K}.
\end{align*}
\end{theorem}

\begin{theorem}\label{thm:MG_AMOO_GD_Lip}
Suppose that $\{ f_i(\x) \}_{i\in [m]}$ are $G$-Lipschitz. \MGAMOO~(Algorithm~\ref{alg:MGAMOO-Reduction}) with Online Gradient Descent method as \texttt{SOO} guarantees that after $K$ iterations the worst-case gap is upper bounded by:
\begin{align*}
    \max_{i\in[m]} f_i(\bar{\x}) - f_i^\star \leq \frac{3}{2} \frac{G \enorm{\x_1 - \x_\star} }{K}.
\end{align*}    
\end{theorem}
\begin{proof}[Proof of Theorem~\ref{thm:MG_AMOO_GD_Lip}]
    By combining Theorem~\ref{thm:ogd}, and Corollary~\ref{col:online_reduction} we obtain the theorem.
\end{proof}

\section{Results of $\epsilon$-Approximate AMOO }\label{app:approximate_amoo}

References to the main results of this section, in which we provide the convergence of \PAMOO\ and variations of \MGAMOO\ for the $\epsilon$-approximate AMOO, is presented in Table~\ref{tab:approximate_results}.

\begin{table}[t]
\centering
\renewcommand{\arraystretch}{1.25}  
\begin{tabular}{|>{\centering\arraybackslash}p{5cm}|>{\centering\arraybackslash}p{4.5cm}|}
\hline
\rowcolor{gray!50} \makecell{Algorithm} & \makecell{Reference to Theorem} \\
\hline
 \makecell{\MGAMOO\ w Polyak or GD} &   Theorem~\ref{thm:formal_convex_app_convex_amoo}\\
\hline
 \makecell{\PAMOO} & Theorem~\ref{thm:epsilon_original_pamoo_convergence} \\
\hline
 \makecell{\MGAMOO\ w Online Learner} &  Theorem~\ref{theorem:approximate_online_learner}\\
\hline
\end{tabular}
\vspace{0.1cm}
\caption{References to Theorems for the $\epsilon$-Approximate AMOO setting.
}
\label{tab:approximate_results}
\end{table}

The following definition is a formal version of the informal one in \eqref{eq:Ce_set}.
\begin{definition}[$\epsilon$-Approximate Solution Set] \label{def:epsilon_app_solution_set}
     Let $\epsilon \geq 0$. $\mC_\epsilon$ is an $\epsilon$-Approximate Solution Set ($\epsilon$-ASS) if for every $i \in [m]$ it holds that $f_i(\x) - f_i^\star \leq \epsilon$, i.e.
    \begin{align*}
        \mC_\epsilon = \{ \x \in \reals^n | \ f_i(\x) - f_i^\star \leq \epsilon ~~ \forall i\in[m] \}.
    \end{align*}
\end{definition}

Every $\epsilon$-ASS has the following helpful attribute as described by the following lemma.
\begin{lemma}\label{lemma:epsilon_app_solution_set}
    Let $\epsilon > 0$. For every point $\x \notin \mC_\epsilon$ it holds that the $MG(\x)$ is greater than $\epsilon$, i.e. 
    \begin{align*}
        \MG(\x) = \max_{i\in [m]} \brac{f_i(\x) - f_i^\star}  > \epsilon.
    \end{align*}
\end{lemma}
\begin{proof}[Proof of Lemma~\ref{lemma:epsilon_app_solution_set}]
    We prove by contradiction. Assume that exist $\x \notin \mC_\epsilon$ such that $\max_{i\in[m]} f_i(\x) - f_i^\star \leq \epsilon$. Then, it holds that $f_i(\x) - f_i^\star \leq \epsilon ~~ \forall i\in[m]$, which mean $\x \in \mC_\epsilon$ by definition.
\end{proof}

The following theorems are the formal version of Theorem~\ref{thm:convex_app_convex_amoo}.

\begin{theorem}[Convergence in Convex $\epsilon$-Approximate AMOO]\label{thm:formal_convex_app_convex_amoo}
    \PAMOO~(Algorithm~\ref{alg:PAMOO_original}), and \MGAMOO~(Algorithm~\ref{alg:MGAMOO-Reduction}) with Polyak or GD method (with mild changes) as \texttt{SOO} guarantees that after $K$ iterations, the worst-case gap is upper bounded by:
    \begin{align*}
        (1) \quad & \max_{i\in[m]} f_i(\bar{\x}) - f_i^\star \leq \frac{ G \enorm{\x_1 - \x_\star} }{\sqrt{K}} + \epsilon, \quad \text{when $\{ f_i(\x) \}_{i\in [m]}$ are $G$-Lipschitz.} \\
        (2) \quad & \max_{i\in[m]} f_i(\bar{\x}) - f_i^\star \leq \frac{2\beta \enorm{\x_1 - \x_\star} }{K} + 2\epsilon, \quad \text{when $\{ f_i(\x) \}_{i\in [m]}$ are $\beta$-smooth.}
    \end{align*}    
\end{theorem}
\begin{proof}[Proof of Theorem~\ref{thm:formal_convex_app_convex_amoo}] It holds directly from Theorem~\ref{thm:epsilon_original_pamoo_convergence}, Theorem~\ref{thm:epsilon_MG_AMOO_Polyak}, and Theorem~\ref{thm:epsilon_MG_AMOO_GD_smooth}.
    
\end{proof}
 
\subsection{Auxiliary Lemma for Polyak step}

The following lemma is an $\epsilon$-approximate version of Lemma~\ref{lemma:upper_bound_over_trajectory_gap}.
\begin{lemma}\label{lemma:epsilon_upper_bound_over_trajectory_gap}
Let $\x_\epsilon \in \mC_\epsilon$. Let $I(k) \in \argmax_{i\in[m]} f_i(\x_{k}) - f_i^\star$, and the trajectory $\{ \x_k \}_{k=1}^K$.  Suppose for every $k$ the followings hold:
\begin{align}
    \enorm{\x_{k+1} - \x_\epsilon}^2 \leq \enorm{\x_k - \x_\epsilon}^2 - \frac{\brac{ f_{I(k)}(\x_k) - f_{I(k)}^\star - \epsilon}^2}{\enorm{\nabla f_{I(k)}(\x_k)}^2},
    \label{eq:epsilon_conv_GDPolyakAMOO}
\end{align}
and $f_{I(k)}(\x_k) - f_{I(k)}^\star - \epsilon \geq 0$.
Then, Any algorithm which returns the trajectory average, $\bar{\x} = \frac{1}{K} \sum_{k=1}^K \x_k$, guarantees the worst-case gap is upper bounded by:
\begin{align*}
    (1) \quad & \max_{i\in[m]} f_i(\bar{\x}) - f_i^\star \leq \frac{ G \enorm{\x_1 - \x_\epsilon} }{\sqrt{K}} + \epsilon,~~~ \quad \text{when $\{ f_i(\x) \}_{i\in [m]}$ are $G$-Lipschitz.} \\
    (2) \quad & \max_{i\in[m]} f_i(\bar{\x}) - f_i^\star \leq \frac{2\beta \enorm{\x_1 - \x_\epsilon} }{K} + 2\epsilon, \quad \text{when $\{ f_i(\x) \}_{i\in [m]}$ are $\beta$-smooth.}
\end{align*}
\end{lemma}
\begin{proof}[Proof of Lemma~\ref{lemma:epsilon_upper_bound_over_trajectory_gap}]
    Let $\x_\epsilon \in \mC_\epsilon$. Now, we separate to two cases.\\
    \emph{Case 1:}
    Since the functions are $G$-Lipschitz, i.e. $\enorm{\nabla f_{i}(\x) } \leq G $ for all $i \in [m]$, and $ \x \in \reals^n$. Then, from \eqref{eq:epsilon_conv_GDPolyakAMOO}, we have
    \begin{align*}
        \brac{f_{I(k)}(\x_k) - f_{I(k)}^\star - \epsilon}^2 \leq  G^2 \brac{\enorm{\x_k - \x_\epsilon}^2 - \enorm{\x_{k+1} - \x_\epsilon}^2} .
    \end{align*}
    Summing over $k$, and using Cauchy–Schwarz inequality, i.e. $\brac{\sum_{i=1}^n a_i}^2 \leq n \sum_{i=1}^n a_i^2$ when $a_i \in \reals$, it holds
    \begin{align*}
        \brac{\sum_{k=1}^K \brac{f_{I(k)}(\x_k) - f_{I(k)}^\star - \epsilon}}^2 & \leq K \sum_{k=1}^K \brac{f_{I(k)}(\x_k) - f_{I(k)}^\star - \epsilon}^2\\
        & \leq   G^2 K \sum_{k=1}^K \brac{\enorm{\x_k - \x_\epsilon}^2 - \enorm{\x_{k+1} - \x_\epsilon}^2} \\
        & =  G^2 \enorm{\x_1 - \x_\epsilon}^2 K.
    \end{align*} 
    Recall that $f_{I(k)}(\x_k) - f_{I(k)}^\star - \epsilon \geq 0$. Then by taking square root on both sides, we have
    \begin{align} \label{eq:epsilon_Lipchitz_general_max_gap}
        \sum_{k=1}^K \brac{f_{I(k)}(\x_k) - f_{I(k)}^\star } \leq  G \enorm{\x_1 - \x_\epsilon} \sqrt{K} + \epsilon K.
    \end{align}
    \emph{Case 2:} Recall that $f_{I(k)}(\x_k) - f_{I(k)}^\star - \epsilon \geq 0$. Since the functions are $\beta$-smooth, then, from Lemma~\ref{lemma:smooth-gradient-norm}, and \eqref{eq:epsilon_conv_GDPolyakAMOO} it holds that
    \begin{align*}
        \brac{f_{I(k)}(\x_k) - f_{I(k)}^\star - \epsilon}^2 
        & \leq  \enorm{\nabla f_{I(k)}(\x_k)}^2 \brac{\enorm{\x_k - \x_\epsilon}^2 - \enorm{\x_{k+1} - \x_\epsilon}^2}\\
        & \leq 2\beta \brac{f_{I(k)}(\x_k) - f_{I(k)}^\star}  \brac{\enorm{\x_k - \x_\epsilon}^2 - \enorm{\x_{k+1} - \x_\epsilon}^2}.
    \end{align*}
    By rearranging, and since $f_{I(k)}(\x_k) - f_{I(k)}^\star \geq 0  $, we have
    \begin{align*}
         f_{I(k)}(\x_k) - f_{I(k)}^\star 
        & \leq 2\beta  \brac{\enorm{\x_k - \x_\epsilon}^2 - \enorm{\x_{k+1} - \x_\epsilon}^2} + 2\epsilon - \frac{\epsilon^2}{f_{I(k)}(\x_k) - f_{I(k)}^\star}\\
        & \leq 2\beta  \brac{\enorm{\x_k - \x_\epsilon}^2 - \enorm{\x_{k+1} - \x_\epsilon}^2} + 2\epsilon.
    \end{align*}
    Then, by summing over $k$, we have
    \begin{align} \label{eq:epsilon_smooth_general_max_gap}
         \sum_{k=1}^K f_{I(k)}(\x_k) - f_{I(k)}^\star \leq 2\beta \enorm{\x_1 - \x_\epsilon}^2 + 2 \epsilon K.
    \end{align}
    By applying Lemma~\ref{lemma:worst_case_gap_avg} over \eqref{eq:epsilon_Lipchitz_general_max_gap}, and ~\eqref{eq:epsilon_smooth_general_max_gap} we obtain the proof.
\end{proof}

\subsection{PAMOO algorithm}

In this section we assume that $\epsilon$ is known in advance. Thus we change the variable $\Delta$ in the optimization problem in \PAMOO~(Algorithm~\ref{alg:PAMOO_original}) to this:
\begin{align}
\label{eq:epsilon_delta_var_in_opt_in_PAMOO}
    \Delta_{k} := \left[  f_{1}(\x_k)-f_{1}^\star - \epsilon,\dots, f_{m}(\x_k)-f_{m}^\star - \epsilon \ \right]
\end{align}

We prove first an auxiliary lemma, and then we prove the theorem.
\begin{lemma}\label{lemma:epsilon_conv_PAMOOv1}
    Let $\x_\epsilon \in \mC_\epsilon$, and $I(k) \in \argmax_{i\in[m]} f_i(\x_{k}) - f_i^\star$. At each iteration $k$ in Algorithm \ref{alg:PAMOO_original} the following holds
    \begin{align*}
        \enorm{\x_{k+1}-\x_\epsilon}^2 &\leq \enorm{\x_k-\x_\epsilon}^2 -  \frac{\brac{f_{I(k)}(\x_k)- f_{I(k)}^\star - \epsilon }^2}{\enorm{\nabla f_{I(k)}(\x_k)}^2} .
    \end{align*}
\end{lemma}
\begin{proof}[Proof of Lemma~\ref{lemma:epsilon_conv_PAMOOv1}]
    Define that $f_{\w}(\x)$ is an weighted function of the functions according to $\w$, i.e. $f_{\w}(\x) = \sum_i w_i f_i(\x)$. Denote that for every $\w\in \reals^m_+$ it holds that $f_{\w}$ is a convex function. Hence, for every $\x,\y \in \reals^n$ it holds that
\begin{align*}
    &f_{\w}(\x) - f_{\w}(\y)  \leq \nabla f_\w(\x)^T(\x-\y). 
\end{align*}
Let $\x_\epsilon \in \mC_\epsilon$. Recall that the update rule is given by $\x_{k+1} = \x_k - \nabla f_{\w_k}(\x_k)$. Then, by using the previous equation, we have
\begin{align}
    \enorm{\x_{k+1}-\x_\epsilon}^2 &=\enorm{\x_k-\x_\epsilon}^2 - 2 \nabla f_{\w_k}(\x_k)^\top\brac{\x_k - \x_\epsilon} + \enorm{\nabla f_{\w_k}(\x_k)}^2 \nonumber \\
    &\leq \enorm{\x_k- \x_\epsilon}^2 - 2 \brac{f_{\w_k}(\x_k)- f_{\w_k} (\x_\epsilon)} +\enorm{\nabla f_{\w_k}(\x_k)}^2 \nonumber \\
    &= \enorm{\x_k- \x_\epsilon}^2 - 2 \brac{f_{\w_k}(\x_k)- f_{\w_k}^\star - f_{\w_k}^\star + f_{\w_k} (\x_\epsilon)} +\enorm{\nabla f_{\w_k}(\x_k)}^2\\
    &= \enorm{\x_k- \x_\epsilon}^2 - 2 \brac{f_{\w_k}(\x_k)- f_{\w_k}^\star - \epsilon \w_k } +\enorm{\nabla f_{\w_k}(\x_k)}^2.
    \label{eq:epsilon_PAMOO_dist_from_x}
\end{align}
When the last inequality is from  Definition~\ref{def:epsilon_app_solution_set}, i.e. it holds that $f_{\w_k} (\x_\epsilon) - f_{\w_k}^\star < \epsilon \w_k $ for every $\x_\epsilon \in \mC_\epsilon$. Note that the update rule of \texttt{PAMOO} with the $\Delta$ from \eqref{eq:epsilon_delta_var_in_opt_in_PAMOO} is 
$$
\w_k\in \argmax_{\w\in \reals^m_+} 2\brac{f_{\w}(\x_k)- f_{\w}^\star - \w \epsilon} - \enorm{\nabla f_{\w}(\x_k)}^2.
$$
Note that $I(k) \in \argmax_{i\in[m]} f_i(\x_{k}) - f_i^\star$ when $I(k)$ is a one-hot vector, $a_k = f_{I(k)}(\x_k)- f_{I(k)}^\star - \epsilon$, and $b_k = \enorm{\nabla f_{I(k)}(\x_k)}^2$. Let $\w(\x_k) = I(k) \frac{a_k}{b_k} \in \reals^m_+$, we can lower bound of the last expression as follows
\begin{align*}
    2 \brac{f_{\w_k}(\x_k)- f_{\w_k}^\star - \epsilon \w_k} - \enorm{\nabla f_{\w_k}(\x_k)}^2 & = \max_{\w\in \reals_+^m} \left[ 2\brac{f_{\w}(\x_k)- f_{\w}^\star - \epsilon \w} - \enorm{\nabla f_{\w}(\x_k)}^2 \right] \\
    & \geq 2\brac{f_{\w(\x_k)}(\x_k)- f_{\w(\x_k)}^\star -  \epsilon \w(\x_k)}  -\enorm{\nabla f_{\w(\x_k)}(\x_k)}^2 \\
    & =   \frac{\brac{f_{I(k)}(\x_k)- f_{I(k)}^\star -\epsilon }^2}{\enorm{\nabla f_{I(k)}(\x_k)}^2}.
\end{align*}
Combining the last equation with \eqref{eq:epsilon_PAMOO_dist_from_x}, we obtain
\begin{align*}
    \enorm{\x_{k+1}-\x_\epsilon}^2 &\leq \enorm{\x_k-\x_\epsilon}^2 -  \frac{\brac{f_{I(k)}(\x_k)- f_{I(k)}^\star -\epsilon}^2}{\enorm{\nabla f_{I(k)}(\x_k)}^2}.
\end{align*}
\end{proof}

\begin{theorem}
\label{thm:epsilon_original_pamoo_convergence}            
    \PAMOO~(Algorithm~\ref{alg:PAMOO_original}) with the change of $\Delta$ from \eqref{eq:epsilon_delta_var_in_opt_in_PAMOO} guarantees that after $K$ iterations, the worst-case gap is upper bounded by:
    \begin{align*}
        (1) \quad & \max_{i\in[m]} f_i(\bar{\x}) - f_i^\star \leq \frac{ G \enorm{\x_1 - \x_\epsilon} }{\sqrt{K}} + \epsilon, \quad \text{when $\{ f_i(\x) \}_{i\in [m]}$ are $G$-Lipschitz.} \\
        (2) \quad & \max_{i\in[m]} f_i(\bar{\x}) - f_i^\star \leq \frac{2\beta \enorm{\x_1 - \x_\epsilon} }{K} + 2\epsilon, \quad \text{when $\{ f_i(\x) \}_{i\in [m]}$ are $\beta$-smooth.}
    \end{align*}    
\end{theorem}

\begin{proof}[Proof of Theorem~\ref{thm:epsilon_original_pamoo_convergence}]
    The theorem holds directly from Lemma~\ref{lemma:epsilon_upper_bound_over_trajectory_gap}, which holds directly from Lemma~\ref{lemma:epsilon_conv_PAMOOv1}. 
\end{proof}

\subsection{Polyak step as \texttt{SOO}}
In this section we assume that $\epsilon$ is known in advance. With this knowledge, we add to \MGAMOO~(Algorithm~\ref{alg:MGAMOO-Reduction}) a stopping condition, which guarantees that at each iteration the step size $\eta_k$ is not negative. The stopping condition is: return $\x_k$ if
\begin{align}
    \max_{i\in[m]} f_i(\bar{\x}) - f_i^\star = f_{I(k)}(\x_k) - f_{I(k)}^\star < \epsilon. \label{eq:stopping_cond}
\end{align} 
A property that directly follows from this condition, by combination with Lemma~\ref{lemma:epsilon_app_solution_set} is that the stopping condition is met, $\x_k \in \mC_\epsilon$ means that the max gap is less than $\epsilon$.\\

Now, we will prove the Polyak part of Theorem~\ref{thm:formal_convex_app_convex_amoo}, namely, run \MGAMOO~(Algorithm~\ref{alg:MGAMOO-Reduction}) with the Polyak method as a \texttt{SOO}.
First, we present the update step in iteration $k$ of the Polyak method in the $\epsilon$-Approximate AMOO, which has the following form:
\begin{align}
    \label{eq:epsilon_polyak_as_soo}
    \text{Set} \quad \eta_k = \frac{f(\x_{k}) - f^\star - \epsilon}{\enorm{\nabla f(\x_k)}^2}, \quad \text{define} \quad \g_k \gets \eta_k \nabla f(\x_k), \quad \text{and return} \quad \x_{k+1} = \x_k - \g_k.
\end{align}

Before we prove the theorem, we prove first auxiliary lemmas. The next lemma is an $\epsilon$-approximate version of Lemma~\ref{lemma:conv_MG_AMOO_Polyak_step}

\begin{lemma}\label{lemma:epsilon_conv_MG_AMOO_Polyak_step}
Let $\x_\epsilon \in \mC_\epsilon$. The Polyak update step, \eqref{eq:epsilon_polyak_as_soo}, has the following attribute for every $k$:
    \begin{align*}
        \enorm{\x_{k+1} - \x_\epsilon}^2 \leq \enorm{\x_k - \x_\epsilon}^2 - \frac{\brac{ f(\x_k) - f^\star - \epsilon}^2}{\enorm{\nabla f(\x_k)}^2} .
    \end{align*}

\end{lemma}
\begin{proof}[Proof of Lemma~\ref{lemma:epsilon_conv_MG_AMOO_Polyak_step}]  Since $f(\x_{k}) - f^\star - \epsilon \geq 0$ for every $k$, it holds that $\eta_k \geq 0$ for every $k$. Since $\x_{k+1} = \x_k -  \eta_k \nabla f(\x_k)$, it holds that 
    \begin{align*}
        \enorm{\x_{k+1} - \x_\epsilon}^2 &= \enorm{\x_k - \eta_k \nabla f(\x_k) - \x_\epsilon}^2 \\
        & = \enorm{\x_k - \x_\epsilon}^2 - 2 \eta_k \nabla f(\x_k)^\top (\x_k - \x_\epsilon) + \eta_k^2 \enorm{ \nabla f(\x_k)}^2\\
        &\leq \enorm{\x_k - \x_\epsilon}^2 - 2 \eta_k \brac{ f(\x_k) - f(\x_\epsilon)} + \eta_k^2 \enorm{ \nabla f(\x_k)}^2.
    \end{align*}
    When the inequality holds from the convexity attribute, i.e. $f(\x_k) - f(\x_\epsilon)  \leq \nabla f(\x_k)^\top (\x_k-\x_\epsilon)$. Since $0 \leq  f(\x_\epsilon) - f^\star \leq \epsilon$ and $\eta_k \geq 0$ it holds that
    \begin{align*}
        \enorm{\x_{k+1} - \x_\epsilon}^2
        &\leq \enorm{\x_k - \x_\epsilon}^2 - 2 \eta_k \brac{ f(\x_k) - f^\star + f^\star - f(\x_\epsilon)} + \eta_k^2 \enorm{ \nabla f(\x_k)}^2\\
        & \leq \enorm{\x_k - \x_\epsilon}^2 - 2 \eta_k \brac{ f(\x_k) - f^\star -\epsilon} + \eta_k^2 \enorm{ \nabla f(\x_k)}^2.
    \end{align*}
    Thus, by applying $\eta_k$ from \eqref{eq:epsilon_polyak_as_soo} we obtain the lemma.
\end{proof}

\begin{theorem}\label{thm:epsilon_MG_AMOO_Polyak}
\MGAMOO~(Algorithm~\ref{alg:MGAMOO-Reduction}) with stopping condition from \eqref{eq:stopping_cond}, and with Polyak method as \texttt{SOO} guarantees that after $K$ iterations the worst-case gap is upper bounded by:
\begin{align*}
        (1) \quad & \max_{i\in[m]} f_i(\bar{\x}) - f_i^\star \leq \frac{ G \enorm{\x_1 - \x_\epsilon} }{\sqrt{K}} + \epsilon, \quad ~~\text{ when $\{ f_i(\x) \}_{i\in [m]}$ are $G$-Lipschitz.} \\
        (2) \quad & \max_{i\in[m]} f_i(\bar{\x}) - f_i^\star \leq \frac{2\beta \enorm{\x_1 - \x_\epsilon} }{K} + 2\epsilon, \quad \text{when $\{ f_i(\x) \}_{i\in [m]}$ are $\beta$-smooth.}
    \end{align*}    
\end{theorem}
\begin{proof}[Proof of Theorem~\ref{thm:epsilon_MG_AMOO_Polyak}]
    Note that the polyak method, \eqref{eq:epsilon_polyak_as_soo}, receive from \MGAMOO\ the function $f_{I(k)}(\x)$ at each iteration $k$. Then, from Lemma~\ref{lemma:epsilon_conv_MG_AMOO_Polyak_step} and the stopping condition, the theorem holds.
\end{proof}

\subsection{GD method for the smooth case as \texttt{SOO}}

In this section we will prove the Gradient Descent (GD) part of Theorem~\ref{thm:formal_convex_app_convex_amoo}, namely, run \MGAMOO\ with the GD method for the $\beta$-smooth case as \texttt{SOO}.
Recall the update step, \eqref{eq:GD_as_soo_smooth}, in iteration $k$ of the GD method, which has the following form:
\begin{align*}
    \text{Set} \quad  0 < \eta \leq \frac{1}{\beta}, \quad \text{define} \quad \g_k \gets \eta \nabla f(\x_k), \quad \text{and return} \quad \x_{k+1} = \x_k - \g_k.
\end{align*}

Now, before we prove the GD part of Theorem~\ref{thm:formal_convex_app_convex_amoo} we prove an auxiliary lemma for it.
\begin{lemma}\label{lemma:epsilon_conv_MG_AMOO_GD_smooth}
    Let $\x_\epsilon \in \mC_\epsilon$, and $\eta = \frac{1}{2\beta} $. The GD update step, \eqref{eq:GD_as_soo_smooth}, has the following attribute for every $k$:
    \begin{align*}
        f(\x_k) - f^\star \leq \frac{1}{2\beta} \brac{\enorm{\x_k - \x_\epsilon}^2 - \enorm{\x_{k+1} - \x_\epsilon}^2} + 2\epsilon.
    \end{align*}
\end{lemma}
\begin{proof}[Proof of Lemma~\ref{lemma:epsilon_conv_MG_AMOO_GD_smooth}]Recall that $\x_{k+1} = \x_k -  \eta \nabla f(\x_k)$, and $\eta > 0$. Since $f(\x)$ is convex, it holds that $f(\x_k) - f(\x_\epsilon)  \leq \nabla f(\x_k)^\top(\x_k-\x_\epsilon)$. By combining this convexity attribute with Lemma~\ref{lemma:smooth-gradient-norm}, it holds that
\begin{align*}
    \enorm{\x_{k+1} - \x_\epsilon}^2 &= \enorm{\x_k - \eta \nabla f(\x_k) - \x_\epsilon}^2 \\
    & = \enorm{\x_k - \x_\epsilon}^2 - 2 \eta \nabla f(\x_k)^\top (\x_k - \x_\epsilon) + \eta^2 \enorm{ \nabla f(\x_k)}^2\\
    &\leq \enorm{\x_k - \x_\epsilon}^2 - 2 \eta \brac{ f(\x_k) - f(\x_\epsilon)} + 2 \beta \eta^2 \brac{ f(\x_k) - f^\star}.
\end{align*}
Since $0 \leq  f(\x_\epsilon) - f^\star \leq \epsilon$ and $\eta_k \geq 0$ it holds that
\begin{align*}
    \enorm{\x_{k+1} - \x_\epsilon}^2 
    &\leq  \enorm{\x_k - \x_\epsilon}^2 - 2 \eta \brac{ f(\x_k) - f^\star + f^\star- f(\x_\epsilon)} + 2 \beta \eta^2 \brac{ f(\x_k) - f^\star}\\
    &\leq  \enorm{\x_k - \x_\epsilon}^2 - 2 \eta \brac{1-\eta \beta} \brac{ f(\x_k) - f^\star}  + 2 \eta \epsilon
\end{align*}

    We obtain the lemma by plugging-in $\eta$, and rearranging the equation.
\end{proof}

\begin{theorem}\label{thm:epsilon_MG_AMOO_GD_smooth}
Suppose that $\{ f_i(\x) \}_{i\in [m]}$ are $\beta$-smooth. \MGAMOO~(Algorithm~\ref{alg:MGAMOO-Reduction}) with GD method with $\eta = \frac{1}{2\beta}$ as \texttt{SOO} guarantees that after $K$ iterations the worst-case gap is upper bounded by:
\begin{align*}
    \max_{i\in[m]} f_i(\bar{\x}) - f_i^\star \leq \frac{2\beta \enorm{\x_1 - \x_\star} }{K} + 2\epsilon.
\end{align*}    
\end{theorem}
\begin{proof}[Proof of Theorem~\ref{thm:epsilon_MG_AMOO_GD_smooth}]
Let $I(k) \in \argmax_{i\in[m]} f_i(\x_{k}) - f_i^\star$, and the trajectory $\{ \x_k \}_{k=1}^K$. Note that the GD method, \eqref{eq:GD_as_soo_smooth}, receive from \MGAMOO\ the function $f_{I(k)}(\x)$ at each iteration $k$. Then, from Lemma~\ref{lemma:epsilon_conv_MG_AMOO_GD_smooth} it holds that for every $k$ we have 
\begin{align*}
    f_{I(k)}(\x_k) - f_{I(k)}^\star \leq \frac{1}{2\beta} \brac{\enorm{\x_k - \x_\star}^2 - \enorm{\x_{k+1} - \x_\star}^2} + 2\epsilon.
\end{align*}
    Summing over $k$, we have that
    \begin{align*} 
        \sum_{k=1}^K f_{I(k)}(\x_k) - f_{I(k)}^\star \leq 2\beta \enorm{\x_1 - \x_\star}^2 + 2\epsilon K.
    \end{align*}
    By applying Lemma~\ref{lemma:worst_case_gap_avg} we obtain the proof.
\end{proof}

\subsection{\MGAMOO\ with Online Learner}

\begin{theorem}[\MGAMOO\ with an Online Learner in $\epsilon$-Approximate AMOO]\label{theorem:approximate_online_learner}
    Let $\{f_{i}\}_{i\in[m]}$ satisfy the $\epsilon$-approximate AMOO assumption. Instantiating the single-objective optimizer \texttt{SOO} in \MGAMOO\ with an online learning algorithm with regret $\mathrm{Regret}(K)$ guarantees $MG(\bar{\x}) \leq \mathrm{Regret}(K) / K +\epsilon$. Namely, $\bar{\x}\in \mC_\epsilon$
\end{theorem}

\begin{proof}[Proof of Theorem \ref{theorem:approximate_online_learner}]
    Assume \texttt{SOO} is an online learner. Recall that an online learner satisfy that
    \begin{align}
    \sum_{k=1}^{K} f_k(\x_k) - f_k(\x) \leq \mathrm{Regret}(K) \label{eq:online_learner_app}
\end{align}
for all $\x$.

Further, by Lemma~\ref{lemma:worst_case_gap_avg} it holds that 
    \begin{align*}
        \MG(\bar{\x}) := \max_{i\in[m]} f_i(\bar{\x}) - f_i^\star \leq \frac{1}{K}\sum_{k=1}^K \brac{f_{I(k)}(\x_k) - f_{I(k)}^\star}.
    \end{align*}
We bound the RHS of this equation. Let $\x_\epsilon\in \mC_\epsilon$. Then
\begin{align*}
    \frac{1}{K}\sum_{k=1}^K \brac{f_{I(k)}(\x_k) - f_{I(k)}^\star}&=\frac{1}{K}\sum_{k=1}^K \brac{f_{I(k)}(\x_k) - f_{I(k)}(\x_\epsilon)}+\frac{1}{K}\sum_{k=1}^K \brac{f_{I(k)}(\x_\epsilon)- f_{I(k)}^\star}\\
    &\leq \frac{1}{K}\sum_{k=1}^K \brac{f_{I(k)}(\x_k) - f_{I(k)}(\x_\epsilon)}+\epsilon\\
    &\leq \frac{\mathrm{Regret}(K)}{K} +\epsilon.
\end{align*}
The second relation holds since $\x_\epsilon\in \mC_\epsilon$. The third relation holds by the online learning guarantee while setting $\x=\x_\epsilon$ in \eqref{eq:online_learner_app}.

\end{proof}